\documentclass[letterpaper]{article} 
\usepackage[]{arxiv}  
\usepackage{times}  
\usepackage{helvet}  
\usepackage{courier}  
\usepackage[hyphens]{url}  
\usepackage{graphicx} 
\usepackage{enumitem}
\usepackage{makecell}
\usepackage{natbib}  
\usepackage{caption} 
\usepackage{algorithm}
\usepackage{algpseudocode}
\usepackage{siunitx}
\usepackage{booktabs}
\usepackage{multicol}
\usepackage{multirow}
\usepackage{tikz, graphicx}
\usepackage{enumitem}
\usepackage{xcolor}           
\newif\ifdraft
\drafttrue                  
\ifdraft
  \newcommand{\sa}[1]{\textcolor{blue}{[#1 \textsc{}]}}
\else
  \newcommand{\sa}[1]{}     
\fi
\ifdraft
  
\else
\fi

\usepackage{etoolbox}        
\usepackage{caption}         
\definecolor{framegray}{gray}{0.1}

\newlength{\pcfigwidth}
\usepackage{subcaption}
\usepackage[export]{adjustbox}
\usepackage{tikz}
\usepackage{ifthen}

\usepackage{amsmath}
\usepackage{amssymb}
\usepackage{amsthm}

\usepackage{algorithm}
\usepackage{algpseudocode}
\usepackage{xcolor}

\theoremstyle{definition}
\newtheorem{definition}{Definition}
\newtheorem{lemma}{Lemma}
\newtheorem{theorem}{Theorem}
\newtheorem{proposition}{Proposition}
\newtheorem{corollary}{Corollary}

\usepackage{newfloat}
\usepackage{listings}

\newcommand{\Tr}{\mathrm{Tr}}

\newcommand{\x}{\mathbf{x}}
\newcommand{\Eh}{\widehat{\mathbb{E}}}
\newcommand{\ch}{\mathrm{ch}}
\newcommand{\pa}{\mathrm{pa}}
\newcommand{\Fn}{F_n}
\newcommand{\Fnc}{F_{nc}}
\newcommand{\Nnc}{N_{nc}}
\newcommand{\Snc}{S_{nc}}
\newcommand{\pr}{p_{r}}
\newcommand{\tn}{t_n}
\newcommand{\Tnn}{T_n}
\newcommand{\Th}{\widehat{T}_n}
\newcommand{\that}{\widehat{t}_n}
\newcommand{\rr}{\boldsymbol{\rho}}
\newcommand{\Hn}{H_n}
\newcommand{\Sset}{\mathcal{S}}
\newcommand{\PC}{\mathsf{PC}}
\newcommand{\R}{\mathbb{R}}

\definecolor{algcomment}{gray}{0.5}
\algrenewcommand{\algorithmiccomment}[1]{%
  \hfill\textcolor{algcomment}{\small// #1}} 

\DeclareCaptionStyle{ruled}{labelfont=normalfont,labelsep=colon,strut=off} 
\floatstyle{ruled}
\newfloat{listing}{tb}{lst}{}
\floatname{listing}{Listing}
\usepackage{amsfonts}       
\usepackage{nicefrac}       
\usepackage{microtype}      
\usepackage{lipsum}
\usepackage{fancyhdr}       
\usepackage{graphicx}       
\graphicspath{{media/}}     

\nocopyright 

\title{A Compositional Theory of Curvature in Probabilistic Circuits}
\author {Hrithik Suresh\equalcontrib\textsuperscript{\rm 1},
    Sahil Sidheekh\equalcontrib\textsuperscript{\rm 2,\rm 3}\thanks{Work done in part while Sahil was an intern at AT\&T.}, 
    Shelar Parth Vijay \textsuperscript{\rm 1},
    Yasir Z \textsuperscript{\rm 1},\\
    Sriraam Natarajan\textsuperscript{\rm 3},
    Narayanan Chatapuram Krishnan\textsuperscript{\rm 1}
    }
\affiliations{
    \textsuperscript{\rm 1} Mehta Family School of Data Science and Artificial Intelligence, Department of Data Science, Indian Institute of Technology Palakkad, Kerala, India\\
    \textsuperscript{\rm 2} AT\&T CDO \\
     \textsuperscript{\rm 3} Erik Jonsson School of Engineering \& Computer Science, The University of Texas at Dallas, Richardson, TX, USA

}

\usepackage{bibentry}

\begin{document}

\maketitle

\begin{abstract}
Probabilistic Circuits (PCs) are generative models that support exact inference and, unlike deep neural networks, admit an exact and tractable measure of loss-surface curvature: the trace of the Hessian of the log-likelihood. Recent work regularizes this trace globally to bias learning toward flatter, better generalizing optima. We show that treating sharpness as a global regularizer can be misspecified for PCs, whose curvature is inherently compositional. We prove that each sum node's contribution to the Hessian trace factorizes exactly into its circuit flow, which measures how heavily the node is used, and a local sharpness term determined by its output distribution. This decomposition provides insights into why global sharpness regularization is depth biased and can lead to underfitting. Building on it, we introduce an adaptive sharpness aware regularizer that penalizes nodes based on intrinsic local curvature and preserves closed form EM updates. We also show that empirically, this targeted regularization recovers the generalization that global regularization sacrifices while retaining the robustness and benefits of sharpness aware learning. 
\end{abstract}

\section{Introduction}
Probabilistic circuits (PCs) are generative models whose structural constraints make a broad class of probabilistic queries exactly tractable and computable in time linear in the circuit size~\cite{ChoiIMS2020}. This distinguishes PCs from mainstream deep generative models such as GANs~\cite{GoodfellowNeurIPS2014}, VAEs~\cite{KingmaICLR2014}, and normalizing flows~\cite{PapamakariosJMLR2021}, and makes them particularly useful in settings where exact and reliable inference is required, including constrained generation~\cite{zhang2023tractable}, image inpainting~\cite{liu2024image}, robust representation learning~\cite{braun2025tractable}, multimodal fusion~\cite{sidheekh2025credibilityaware}, neuro-symbolic AI~\cite{ahmed2022semantic,karanam2025unified} and algorithmic recourse~\cite{nemecekICLR2025,sabuTAI2026}, among others. As PCs become deeper and more expressive, however, they are increasingly prone to overfitting in low-data regimes and may converge to sharp optima that generalize poorly \cite{SureshAAAI2026}.

Sharpness-aware learning addresses this problem by biasing optimization toward flatter regions of the loss landscape~\cite{ForetICLR2021}. In deep neural networks, the relevant Hessian-based curvature is generally intractable and must be estimated through stochastic approximations~\cite{KaurPMLR2023,SankarAAAI2021}. Recent work ~\cite{SureshAAAI2026} has shown that PCs are a rare exception in this context: their structure makes the Hessian trace of the negative log-likelihood exactly computable in a single forward--backward pass, enabling a tractable global trace regularizer that improves generalization when data are scarce. While this result established that sharpness can be computed and penalized efficiently in PCs, it {\em treats the trace as a single global measure of flatness}.

In this work, we ask what this global curvature actually represents inside a PC, and whether it is always an appropriate learning signal. Our motivation comes from three empirical observations. First, in the high-data regime, global trace regularization reaches flatter optima but lowers both training and test log-likelihood (as illustrated in Figure~\ref{fig:landscape}), indicating under-fitting rather than improved generalization. Second, the node-wise contributions to the trace are highly concentrated, with a small fraction of sum nodes accounting for most of the total curvature. Third, restricting regularization to the largest-contributing nodes degrades performance further. These observations suggest that the issue is not only the \emph{overall strength} of regularization, but how it is \emph{allocated} across the circuit.
Our central thesis is that these quantities separate exactly, and that separating them helps resolve the puzzle. Specifically, we show that a sum node's contribution to the Hessian trace factorizes \emph{exactly} into two semantically distinct terms:
\begin{equation*}
\Tnn(\mathbf{x}) = \underbrace{\Fn(\mathbf{x})^2}_{\text{contextual usage}} \cdot \underbrace{\tn(\mathbf{x})}_{\text{local curvature}},
\end{equation*}
where $\Fn$ is the circuit flow through node $n$, measuring how strongly the node is used in explaining the input, while $\tn(\mathbf{x})$ is a local curvature measure that depends only on the node $n$. Thus, \textbf{sharpness in a PC is local in origin and contextual in its global effect}. A node may contribute strongly to the Hessian trace because it is locally sharp, because it lies on a high-flow path, or because both effects coincide. Consequently, ranking nodes by their global contribution need not identify the nodes with the largest intrinsic local curvature.

\begin{figure}[t]
    \centering
    \includegraphics[width=\linewidth]{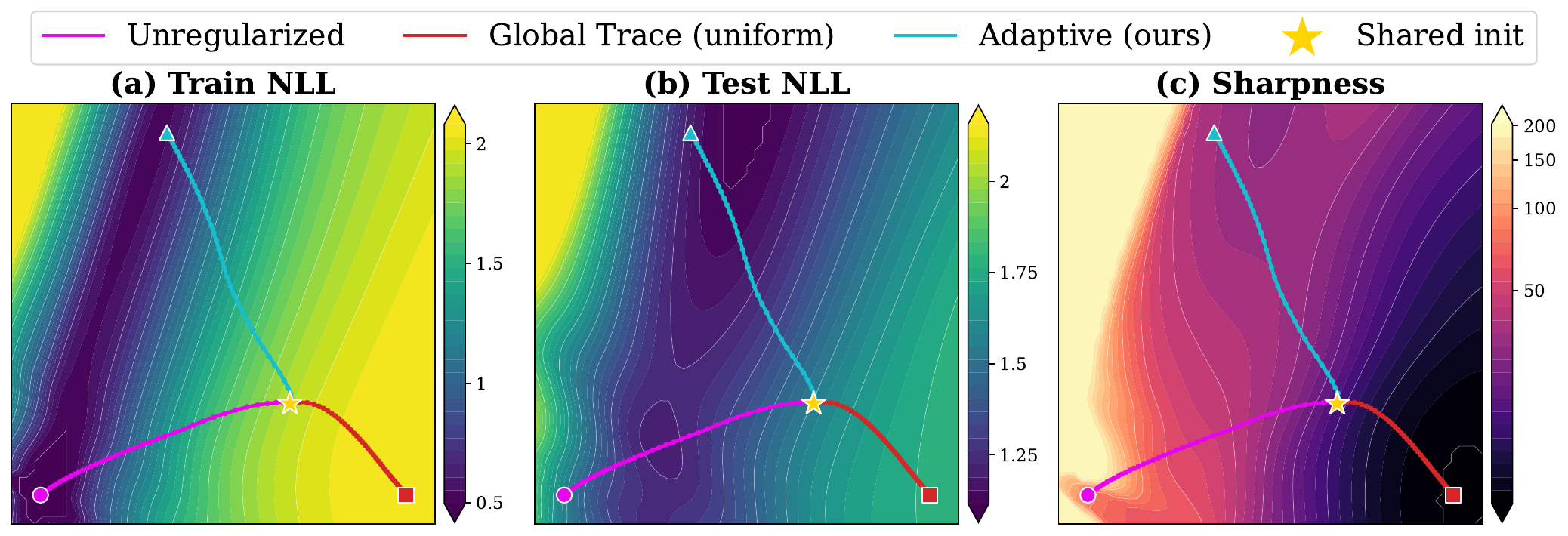}
    \caption{ \textbf{Global Flatness Can Underfit: Local Curvature Determines Where to Regularize.} Training trajectories of a PC on a 2D data distribution, from a shared initialization (star), shown over the (a) train NLL, (b) test NLL, and (c) sharpness surface. The unregularized model reaches a sharp optimum that generalizes poorly. The global trace regularizer's uniform penalty helps achieve flattest region, but at the expense of under-fitting (poor train and test NLL), whereas our adaptive penalty targets nodes of high \emph{local} curvature $\tn$ and generalizes best, while also reducing sharpness. Both approaches use the same regularization strength $\mu$.}
    \label{fig:landscape}
\end{figure}

We first characterize and study both factors theoretically. 
For the local term, we show that the Hessian of a sum node is rank one and that its unique nonzero eigenvalue is exactly $\tn$. Consequently, $\tn$ fully captures the node's ambient second-order geometry. 
For the contextual term, we build on the standard circuit-flow recursion to analyze how structural usage shapes global curvature, and show that it causes a depth bias in the global trace toward shallow, high-flow nodes. We formalize this and derive an exact condition under which rankings by global contribution $\Tnn$ and local curvature $\tn$ disagree, explaining why selecting nodes by their global trace contribution can preferentially target highly used circuit components rather than intrinsically sharp mixtures.
Finally, we show how this decomposition can be algorithmically useful. Since $\Tnn$ conflates local curvature with structural usage, it is not a reliable quantity for deciding where regularization should act. We propose an adaptive sharpness regularizer use the empirical local curvature $\widehat{t}_n$ to gate the existing trace penalty, assigning stronger regularization to intrinsically sharper mixture nodes while preserving the tractable updates and linear-time complexity of the global method. The resulting learner retains the low-data gains of sharpness-aware training while reducing the under-fitting induced by uniform regularization at higher data scales.




\section{Background and Preliminaries}
\label{sec:background}

\subsection{Probabilistic Circuits}
\label{sec:pc}

A probabilistic circuit (PC) over random variables $\mathbf{X}=\{X_1,\dots,X_d\}$ is a rooted directed acyclic graph in which each node $n$ computes a distribution $p_n$, defined recursively: 
\begin{equation} p_n(\mathbf{x}) = \begin{cases} 
f_n(\mathbf{x}), & n \text{ an input node,} 
\\ \prod_{c\in \ch(n)} p_{c}(\mathbf{x}), & n \text{ a product node,} 
\\ \sum_{c\in \ch(n)} \theta_{nc}\, p_c(\mathbf{x}), & n \text{ a sum node,} \end{cases} 
\label{eq:pc}
\end{equation} 
where $f_n$ is a univariate input distribution, $\ch(n)$ are the children of $n$, and the sum-node weights satisfy $\theta_{nc}\ge 0$ and $\sum_{c\in\ch(n)}\theta_{nc}=1$, so that the likelihood of a sample is simply the value computed at the root, $p(\mathbf{x})=\pr(\mathbf{x})$. 
Product nodes represent factorizations over disjoint scopes (decomposability), while sum nodes represent convex mixtures over children with identical scope (smoothness). Together, these two structural conditions are what make a broad class of probabilistic queries tractable, computable exactly and in time linear in circuit size~\cite{ChoiIMS2020}. This formalism is general enough to subsume arithmetic circuits~\cite{darwiche2003differential-ac}, sum-product networks~\cite{PoonICCV2011}, PSDDs~\cite{KisaKR2014}, and cutset networks~\cite{rahman2014cutset} as special cases. Throughout, we write $\Sset$ for the set of sum nodes and collect their weights into $\boldsymbol{\theta}=\{\theta_{nc}\}$. Unless stated otherwise, we treat these sum-edge weights as untied free parameters, and measure curvature w.r.t this representation.

\textbf{Circuit Flows and EM Learning} 
\label{sec:flows}.
While the upward evaluation of a PC computes the likelihood at the root, analyzing learning and curvature requires attributing this computation to individual nodes and edges. Circuit flows \cite{LiuNeurIPS2021} provide this downward attribution by measuring
how much probability mass the root node routes through a given node when explaining a particular sample, and it is this quantity that will allow us later to localize curvature to specific parts of the circuit. Flow propagates from the root downward, and at each node its value is determined by the type of the node's parents. 
Setting $F_r(\mathbf{x})=1$ at the root, the flow at node $n$ is 
\begin{equation*} \Fn(\mathbf{x}) = \sum_{\substack{m\in\pa(n)\\ m \text{ is product}}} F_m(\mathbf{x}) \ + \sum_{\substack{m\in\pa(n)\\ m \text{ is sum}}} F_m(\mathbf{x})\,\theta_{mn}\, \frac{p_n(\mathbf{x})}{p_m(\mathbf{x})}, \label{eq:flowrec} 
\end{equation*}
meaning a product parent passes its flow to each child undivided, whereas a sum parent attenuates the flow it passes to child $n$ by the routing responsibility $\theta_{mn}\,p_n(\mathbf{x})/p_m(\mathbf{x})$. 
Equivalently, $\Fn(\mathbf{x}) = \partial \log \pr(\mathbf{x})/\partial \log p_n(\mathbf{x})$: the flow is the sensitivity of the root log-likelihood to the node's log-output. 
Flow at an edge $(n,c)$ leaving a sum node $n$ is 
\begin{equation} 
\Fnc(\mathbf{x}) = \theta_{nc}\,\frac{p_c(\mathbf{x})}{p_n(\mathbf{x})}\,\Fn(\mathbf{x}), 
\label{eq:edgeflow} 
\end{equation} 
and from this the gradient of the log-likelihood with respect to the corresponding sum weight follows directly, $\partial \log \pr(\mathbf{x})/\partial\theta_{nc}=\Fnc(\mathbf{x})/\theta_{nc}$. 
We also write $r_{nc}(\mathbf{x})=\theta_{nc}\,p_c(\mathbf{x})/p_n(\mathbf{x})$ for the posterior routing responsibility of child $c$, which satisfies $\sum_{c\in\ch(n)} r_{nc}(\mathbf{x})=1$ and so behaves as a proper distribution over $n$'s children. Node and edge flows are obtained together from a single forward--backward pass over the circuit, and this same pass suffices to compute gradients and closed-form expectation-maximization (EM) updates in linear time~\cite{LiuNeurIPS2021}. In the EM view, the edge flow $\Fnc(\mathbf{x})$ is the expected number of times sample $\mathbf{x}$ traverses edge $(n,c)$, and the M-step simply aggregates these expectations across a dataset. 

\subsection{Sharpness and Regularization in PCs}
\label{sec:trace}

Recent works have pushed PCs toward deeper, more expressive, and more scalable architectures while preserving tractable inference~\cite{PeharzUAI2019,PeharzICML2020,SidheekhPMLR2023,sidheekh2024building,LiuICML2024,sidheekh2026geometryaware}. This increased capacity, however, also makes them more susceptible to overfitting, especially in limited-data regimes. 
Existing approaches have attempted to exploit the circuit's tractable structure to address this issue via dropout \cite{ventola2023probabilistic}, parameter smoothing, data softening, and entropy regularization~\cite{LiuNeurIPS2021}. 
A complementary view studies overfitting through the geometry of the loss landscape. Flat minima have long been associated with better generalization~\cite{hochreiter1997flat,KeskarICLR2017}, motivating methods such as sharpness-aware minimization~\cite{ForetICLR2021}. This relationship is however not universal: sharpness can depend on the chosen parameterization, and its connection to generalization varies across models and tasks~\cite{dinh2017sharp,andriushchenko23a}. More recently, \cite{SureshAAAI2026} brought this perspective to PCs, showing that overfitting can coincide with convergence to sharp optima and that the Hessian trace of the negative log-likelihood can be computed exactly for general PCs.
They showed that as $\pr$ is multilinear in each sum-edge weight, the diagonal of the Hessian of the log-likelihood coincides with the squared first derivative along that weight, so the trace of the Hessian reduces to a sum of squared edge gradients. 
Writing the per-example negative log-likelihood (NLL) as $\ell(\boldsymbol{\theta};\mathbf{x})=-\log \pr(\mathbf{x})$, the trace  \begin{equation} \Tr\!\left(\nabla^2_{\boldsymbol{\theta}}\, \ell(\boldsymbol{\theta};\mathbf{x})\right) = \sum_{n\in\Sset}\sum_{c\in\ch(n)}\left(\frac{\Fnc(\mathbf{x})}{\theta_{nc}}\right)^2 \;\ge\; 0. \label{eq:trace} 
\end{equation} 
quantifies the \textbf{sharpness} of the loss landscape at $\boldsymbol{\theta}$, i.e. curvature information, and is computable exactly and in time linear in circuit size. Overfitting in PCs is typically correlated with convergence to sharp optima, and incorporating this trace as a constraint in the EM M-step yields a closed-form, trace-regularized update for each sum weight, 
\begin{equation} 
\theta_{nc} = \frac{\Nnc + \sqrt{\Nnc^2 + 4\lambda\mu\,\Nnc}}{2\lambda}, 
\label{eq:emupdate}
\end{equation} 
where $\Nnc=\sum_i \Fnc(\mathbf{x}_i)$ is the batch-aggregated expected count and $\lambda,\mu\ge0$ are the Lagrange multipliers enforcing, respectively, the simplex constraint and the regularization strength. This enables convergence to flatter optima that can generalize better. However, in this framework a single strength $\mu$ is applied uniformly to every sum node, we refer to this method as the \emph{global trace regularizer}, and we take it as our baseline and point of departure.
To reason about \emph{where} in the circuit this trace originates, rather than only its aggregate value, we attribute it to individual sum nodes.
The sample-wise \emph{global trace contribution} of a sum node $n$ collects the terms of Eq.~\eqref{eq:trace} belonging to its outgoing edges, 
\begin{equation}
\Tnn(\mathbf{x}) = \sum_{c\in\ch(n)}\left(\frac{\Fnc(\mathbf{x})}{\theta_{nc}}\right)^2,  \  \Th = \frac{1}{N}\sum_{i=1}^{N}\Tnn(\mathbf{x}_i) 
\label{eq:Tn} 
\end{equation}
so that $\sum_{n\in\Sset}\Tnn(\mathbf{x})$ recovers the trace in Eq.~\eqref{eq:trace} exactly, and $\Th$ is its empirical average at node $n$. The collection $\{\Th\}$ thus gives a node-level view of where the model's sharpness actually resides, we analyze this in the next section.

\section{When Global Sharpness Fails}
\label{sec:hook}

The global trace regularizer was motivated by the observation that reducing sharpness improves generalization  by enabling convergence to flatter optima~\cite{SureshAAAI2026}. However, as we show in this section, treating sharpness as a single global measure obscures how curvature is distributed across the circuit and which components give rise to it. We identify three empirical observations
that motivate a more structured view of sharpness and its regularization.

\begin{table}[t]
\centering
\caption{Mean percentage change in test NLL ($\Delta$NLL), reduction in degree of overfitting ($\Delta$DoF), and decrease  in the loss-surface sharpness ($\Delta$Sharp) across 20 DEBD binary datasets, achieved by global-trace regularized HCLT, compared to non-regularized baseline, averaged over 5 runs.}
\label{tab:mean_improvements}
\small
\setlength{\tabcolsep}{6pt}
\renewcommand{\arraystretch}{1.1}
\begin{tabular}{@{}crrr@{}}
\toprule
\textbf{Data \%} & \textbf{$\Delta$Test-NLL} & \textbf{$\Delta$DoF} & \textbf{$\Delta$Sharp} \\
\midrule
$1\%$   & $7.12$ & $8.81$ & $23.04$ \\
$5\%$   & $2.19$ & $5.79$ & $15.32$ \\
$10\%$  & $0.27$ & $2.73$ & $10.42$ \\
$50\%$  & $-0.48$ & $0.21$ & $5.55$ \\
$100\%$ & $-0.31$ & $0.06$ & $4.04$ \\
\bottomrule
\end{tabular}
\end{table}


\begin{figure}[ht]
    \centering
    \includegraphics[width=\linewidth]{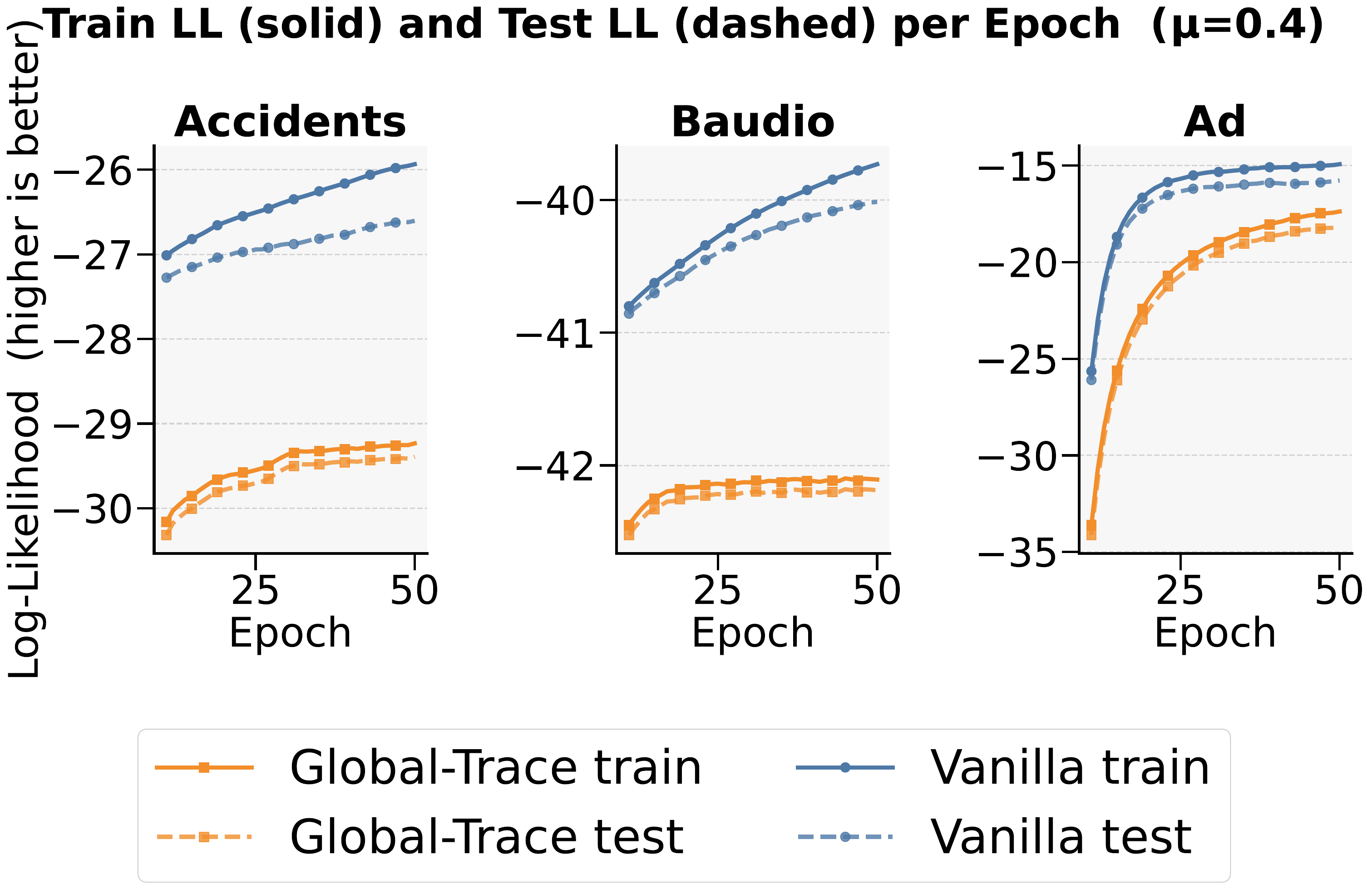}
    \caption{
    Training and test log-likelihood trajectories for the vanilla and global trace-regularized models on \texttt{accidents}, \texttt{baudio}, and \texttt{ad} using the full training data. Global trace regularization reduces the train-test gap but lowers both training and test likelihood, indicating underfitting rather than improved generalization.
    }
    \label{fig:underfitting}
\end{figure}

\paragraph{Observation 1: A flatter optimum need not be a better one.}
Table \ref{tab:mean_improvements} shows that the global trace regularizer helps reach flatter optima but can nevertheless attain lower test log-likelihood than the unregularized model when sufficient data are available. This degradation is not explained by a widening train-test gap, as the corresponding train log-likelihood also decreases, as shown in Figure~\ref{fig:underfitting}. It therefore drives learning toward solutions that are flatter yet fit both the training and test distributions less well, indicating underfitting. Reducing overall curvature is thus not sufficient, understanding how that curvature is important to preserve model capacity.


\begin{figure}[ht]
\centering
\includegraphics[width=\linewidth]{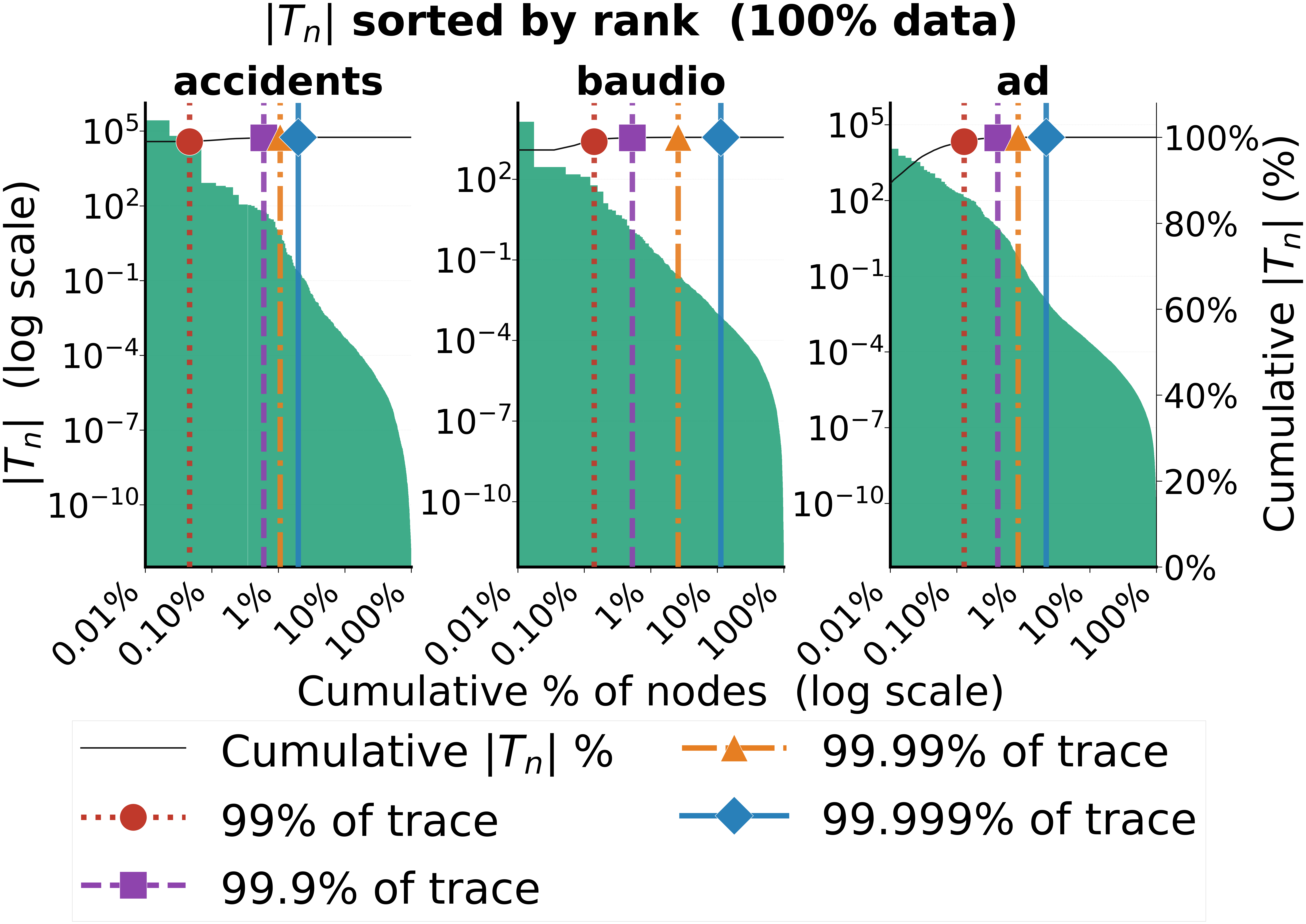}
\caption{\textbf{Only a small fraction of sum nodes accounts for most of the global Hessian trace.}
Cumulative share of the total trace as sum nodes are ranked by their empirical contribution $(\Th)$. Across datasets, the distribution is highly concentrated: a small set of nodes dominates the global sharpness signal, while most nodes contribute negligibly.}
\label{fig:global-contentration}
\end{figure}

\paragraph{Observation 2: Global sharpness is concentrated in a small set of nodes.}
Examining how the global Hessian trace is distributed across the circuit reveals that the node-wise contributions $(\Th)$ defined in Eq~\eqref{eq:Tn} are highly nonuniform. As shown in Figure~\ref{fig:global-contentration}, a small minority $(<10\%)$ of sum nodes accounts for nearly the entire trace $(99.99\%)$, while the contributions of other nodes are negligible. Global sharpness is therefore not spread evenly across the model, but concentrated in a small set of circuit components. This also reveals {\bf a limitation of using a single global regularization strength}. The regularizer penalizes all sum nodes equally despite their markedly different contributions to the total curvature, and cannot distinguish the few nodes that dominate the sharpness signal from the many nodes that are already effectively flat. 

\begin{figure}
\centering
\includegraphics[width=0.9\linewidth]{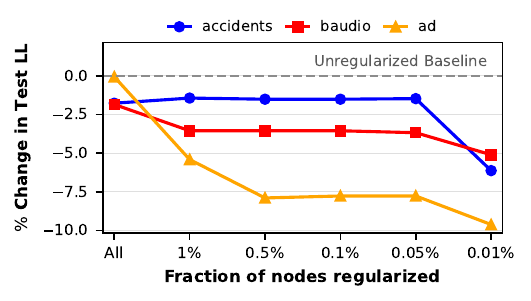}
\caption{\textbf{Targeting the largest global trace contributors worsens underfitting.}
Change in test log-likelihood relative to the unregularized model when trace regularization is applied to all sum nodes or restricted to progressively smaller fractions ranked by $\widehat{T}_n$. Selecting fewer top-ranked nodes generally worsens performance, showing that global contribution is a poor criterion for allocating regularization.
}
\label{fig:Gtrace_thresh}
\end{figure}

\paragraph{Observation 3: Targeting the largest contributors makes underfitting worse.}
The concentration observed above suggests a natural solution: restrict regularization to the nodes with the largest contributions to the global trace. If underfitting were caused primarily by applying the penalty to many negligible contributors that are already sufficiently flat, this strategy could preserve capacity while retaining curvature control. 
However, as shown in Figure~\ref{fig:Gtrace_thresh}, the test log-likelihood degrades further as the penalty is restricted to progressively fewer top-$(\Th)$ nodes. Thus, although $(\Th)$ measures a node's contribution to global sharpness, it does not identify where regularization should act.
Taken together, these observations suggest that neither greater flattening nor restricting regularization to fewer high-contribution nodes is sufficient to resolve the observed degradation. They instead raise a more fundamental question: what does a node’s global trace contribution actually measure, and why can it be a poor signal for allocating regularization? 

\section{A Compositional Theory of Sharpness}
\label{sec:theory}

To address the above puzzle, we first show theoretically that a node's global trace contribution factorizes exactly into a structural factor and a local curvature factor. We characterize them geometrically, and derive the conditions under which the two induce different node rankings. Together, these results formalize the thesis that curvature in a PC is generated locally and expressed globally through circuit flow.

\subsection{Exact Global-Local Decomposition}
\label{sec:decomp}

\begin{figure*}[!ht]
    \centering
    \begin{subfigure}{0.245\linewidth}
        \includegraphics[width=\linewidth]{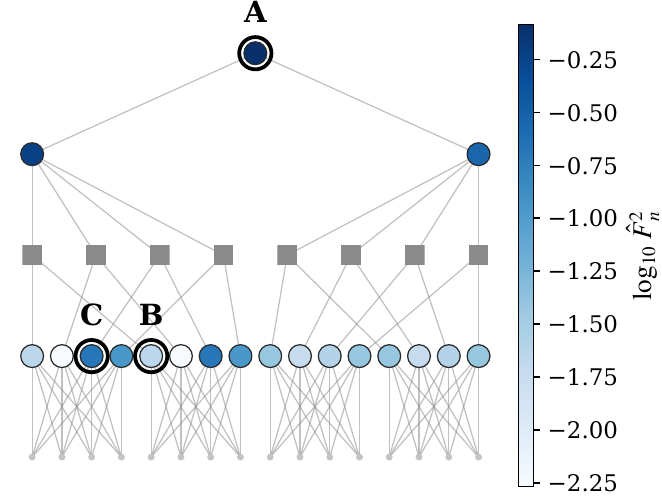}
        \caption{Contextual Usage $(F_n)$}
        \label{fig:decomp-a}
    \end{subfigure}
    \hfill
    \begin{subfigure}{0.245\linewidth}
        \includegraphics[width=\linewidth]{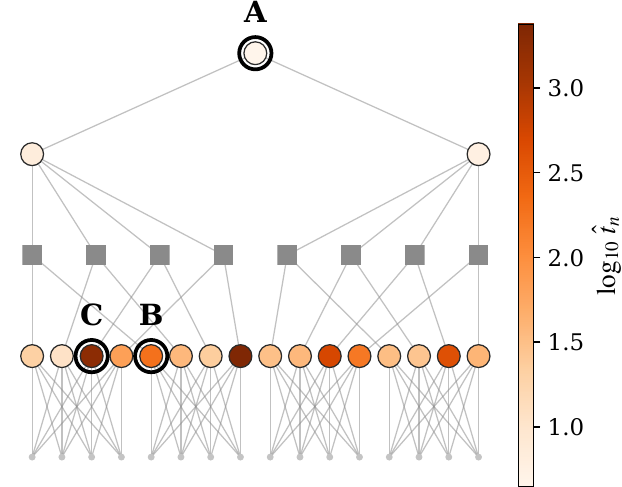}
        \caption{Intrinsic local curvature $(t_n)$}
        \label{fig:decomp-b}
    \end{subfigure}
    \hfill
    \begin{subfigure}{0.245\linewidth}
        \includegraphics[width=\linewidth]{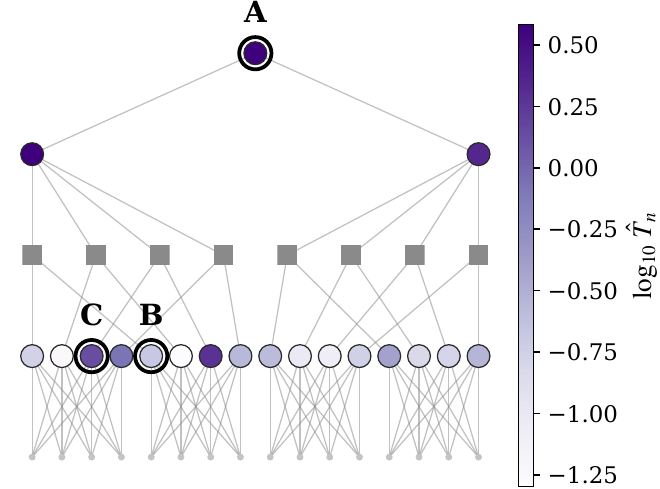}
        \caption{Global trace contribution $(T_n)$}
        \label{fig:decomp-c}
    \end{subfigure}
    \hfill
    \begin{subfigure}{0.245\linewidth}
        \includegraphics[width=\linewidth]{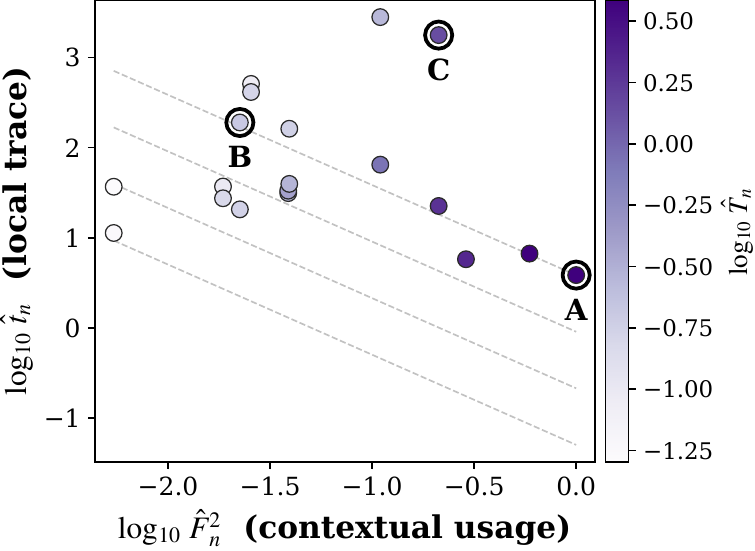}
        \caption{Flow-curvature interaction}
        \label{fig:decomp-d}
    \end{subfigure}

   \caption{\textbf{Compositional structure of sharpness in a trained probabilistic circuit.} For each input, a sum node's global trace contribution factorizes as $\Tnn=\Fn^2\tn$, separating local curvature from its amplification through circuit flow. Panels~(a)--(c) show the same circuit nodes colored by empirical contextual usage, local curvature, and global contribution, while panel~(d) visualizes their interaction, showing that nodes with the large global curvature contribution need not have the large local curvature.}
    \label{fig:curvature-decomposition}
\end{figure*}

\begin{theorem}[]\label{thm:decomp}
Consider any sum node \(n\) with strictly positive outgoing weights in a smooth, decomposable PC. For any input \(\mathbf{x}\), its global trace contribution $\Tnn(\mathbf{x})$
factorizes as
$
\Tnn(\mathbf{x})
=
\Fn(\mathbf{x})^2\,\tn(\mathbf{x}),
$
$
\tn(\mathbf{x})
=
\sum_{c\in\ch(n)}
(p_c(\mathbf{x})/p_n(\mathbf{x}))^2
\label{eq:decomp}
$
Consequently,
$
\Tr\left(
\nabla_{\boldsymbol{\theta}}^2
\ell(\boldsymbol{\theta};\mathbf{x})
\right)
=
\sum_{n\in\Sset}
\Fn(\mathbf{x})^2\,\tn(\mathbf{x}).
$
\end{theorem}


The decomposition is exact and follows directly from the edge-flow factorization. Moreover, both factors are available from the same upward-downward computation used to evaluate circuit flows, so computing the decomposition adds no asymptotic cost. 
The two factors, however, have distinct interpretations. The term
\(\Fn(\mathbf{x})^2\) measures the node's \emph{contextual usage}: how strongly the circuit output depends on node \(n\) for input \(\mathbf{x}\). The term
\(\tn(\mathbf{x})\), as we will show, is the node's \emph{local trace}, determined entirely by the outputs of its local mixture. Hence, a large global contribution \(\Tnn\) may arise from large local curvature, large contextual usage, or both. A node that is frequently used can dominate the global trace even when its local mixture is not among the sharpest, while a locally sharp node can contribute little globally if little flow reaches it.
Under uniform penalization, this can suppress high-usage components regardless of their local geometry, leading regularization to act where curvature is globally amplified rather than where it is intrinsically largest.

Figure~\ref{fig:curvature-decomposition} makes the distinction concrete. The nodes that dominate
the global trace in panel~(c) align more closely with high contextual usage in panel~(a) than with
the largest local traces in panel~(b). This mismatch motivates a closer examination of the two
factors: we first characterize the local geometry summarized by $\tn$, and then study how $\Fn$
transports that geometry through the circuit.

\subsection{Rank-One Local Geometry}
\label{sec:rankone}

The local factor $\tn$ is more than a convenient scalar summary, and as we show below, it captures the entire local
second-order geometry of a sum node. Consider a sum node $n$ in isolation, with local negative
log-output
$\ell_n(\boldsymbol{\theta}_n;\mathbf{x})=-\log\big(\sum_{c}\theta_{nc}p_c(\mathbf{x})\big)$,
and write $\rho_{nc}(\mathbf{x})=p_c(\mathbf{x})/p_n(\mathbf{x})$ for the output ratios, collected
into the vector $\rr_n$.

\begin{proposition}[Rank-one local Hessian]\label{prop:rankone}
The gradient and Hessian of $\ell_n$ in the edge weights $\boldsymbol{\theta}_n$ are
$\nabla_{\boldsymbol{\theta}_n}\ell_n = -\rr_n$ and
$\Hn(\mathbf{x}) = \nabla^2_{\boldsymbol{\theta}_n}\ell_n = \rr_n\rr_n^\top$. Hence, $\Hn$ is
positive semidefinite and, whenever $\rr_n\neq 0$, has rank one, with unique nonzero eigenvalue
$\lambda_{\max}(\Hn)=\lVert\rr_n\rVert_2^2=\tn$. Consequently
$\tn = \Tr(\Hn) = \lVert\Hn\rVert_2 = \lVert\Hn\rVert_F$.
\end{proposition}

Thus, the local second-order geometry of a sum node collapses to a single scalar, the quantity $\tn$, which is
simultaneously the local hessian trace, the maximum ambient curvature, and the total Hessian magnitude.

\subsection{Context Propagation and the Depth Bias}
\label{sec:propagation}

The contextual factor $\Fn^2$ is what makes $\Tnn$ and $\tn$ differ, and it induces a systematic \emph{depth bias} in the global trace: because flow is partitioned among children at every sum node on the way down from the root, shallow nodes accumulate more of it than deep ones, so the global contribution $\Tnn=\Fn^2\tn$ favors structurally shallow nodes regardless of their intrinsic curvature. To understand this bias better, recall a standard consequence of the circuit-flow recursion: product nodes transmit flow unchanged, whereas sum nodes scale it by posterior routing responsibilities. Unrolling the flow recursion thus shows that flow is in fact attenuated by the specific edges it traverses, rather than depth as such.

\begin{lemma}
\label{lemma:propagation}
For a tree-structured PC, the node flow is the product of routing responsibilities along the
unique root-to-$n$ path $\pi(r,n)$, taken over its \emph{sum} edges ($E_{sum}$) only:
$\Fn(\mathbf{x}) = \prod_{e\in\pi(r,n)\cap E_{\mathrm{sum}}} r_e(\mathbf{x})$, with
$r_{(m,n')}=\theta_{mn'}\,p_{n'}/p_m$. For a DAG,
$\Fn(\mathbf{x})=\sum_{\pi\in\Pi(r,n)}\prod_{e\in\pi\cap E_{\mathrm{sum}}} r_e(\mathbf{x})$,
summing over all root-to-$n$ paths $\Pi(r,n)$.
\end{lemma}

Two consequences follow. First, product depth does not attenuate flow: only sum edges do, so the notion of ``depth'' relevant to sharpness is the number of upstream sum edges rather than the raw topological depth. Second, in a DAG a shared node can receive flow through several contexts, and converging paths may compensate for path-wise attenuation; flow therefore need not decay monotonically with depth. The following conditional bound however holds along tree paths.

\begin{corollary}
\label{cor:attenuation}
Consider a tree-structured PC. If every sum-edge responsibility on the root-to-$n$ path satisfies $r_e(\mathbf{x})\le\rho<1$, then
$
\Fn(\mathbf{x})\le\rho^{d_\Sigma(n)},
$
and
$
\Tnn(\mathbf{x})\le\rho^{2d_\Sigma(n)}\tn(\mathbf{x}),
\label{eq:conditional-attenuation}
$
where $d_\Sigma(n)$ is the number of upstream sum edges.
\end{corollary}

The corollary formalizes a mechanism in which the contextual factor can suppress a locally sharp node $n$ reached only through many attenuating routing decisions: its curvature is present, but its global trace contribution is discounted geometrically in the number of upstream \emph{sum} edges $d_\Sigma(n)$.
Thus, the global trace can concentrate on shallow, heavily-used nodes, which we also confirm as illustrated in Figure~\ref{fig:depth-contrib}.

\subsection{Global-Local Ranking Reversals}
\label{sec:reversal}
The decomposition also implies that ranking nodes by their global sharpness contribution is not the same as ranking them by their intrinsic local curvature. A simple analysis gives us the  following exact condition for disagreement.
\begin{proposition}
\label{prop:reversal} 
For any two sum nodes $i,j$, $T_i>T_j \iff t_i/t_j > (F_j/F_i)^2$. In particular, a locally less curved node $i$ with $t_i<t_j$ outranks a locally sharper node $j$ whenever $F_i/F_j > \sqrt{t_j/t_i}$.
\end{proposition}
The two orderings answer different questions: ranking by $\Tnn$ identifies \emph{which parameter block strongly affects the root likelihood}, whereas ranking by $\tn$ identifies \emph{which local mixture computation is most sharply curved}. Because $\Tnn$ is amplified by the structural factor $\Fn^2$, which is largest at shallow, heavily-used nodes, the top-$\Tnn$ nodes need not be the locally sharp ones, accounting for the failure of top-$\Tnn$ selection we saw earlier. The ranking reversals are not a theoretical edge case: in the trained circuit of Figure~\ref{fig:curvature-decomposition}, the root node combines maximal usage with the \emph{lowest} local trace among the labeled nodes yet attains the largest global contribution, while a deep node with local trace larger by nearly three orders of magnitude contributes almost nothing globally.

\begin{figure}[ht]
    \centering
    \includegraphics[width=0.9\linewidth]
    {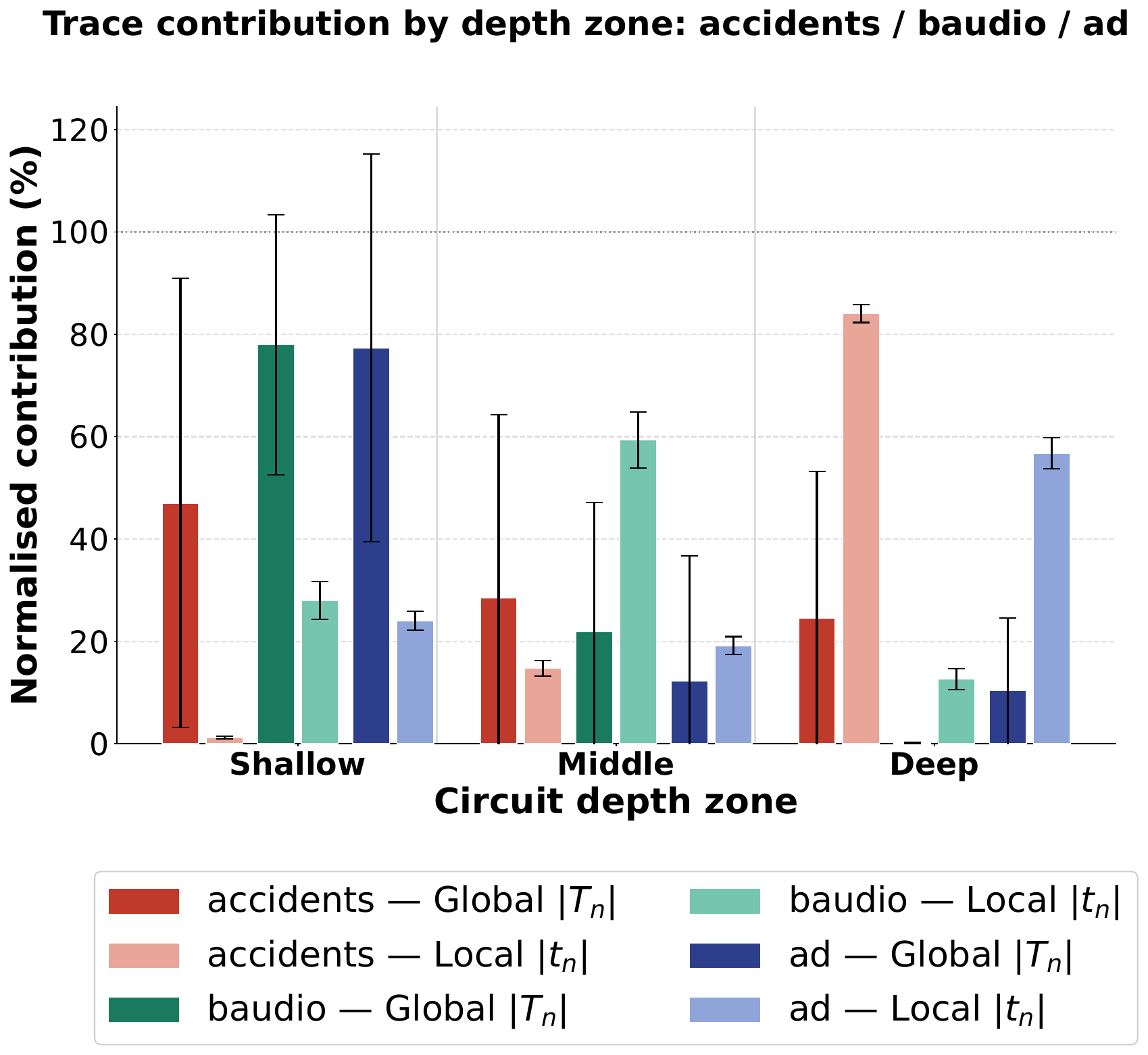}
    \caption{\textbf{Global trace contributions are biased toward early circuit partitions.}
    Normalized global trace contribution across depth partitions on three representative benchmarks. The observed concentration near earlier partitions is consistent with attenuation through upstream sum-node routing, although raw depth alone does not determine flow.}
    \label{fig:depth-contrib}
\end{figure}
\begin{figure}[ht]
    \centering
    \includegraphics[width=0.9\linewidth]
    {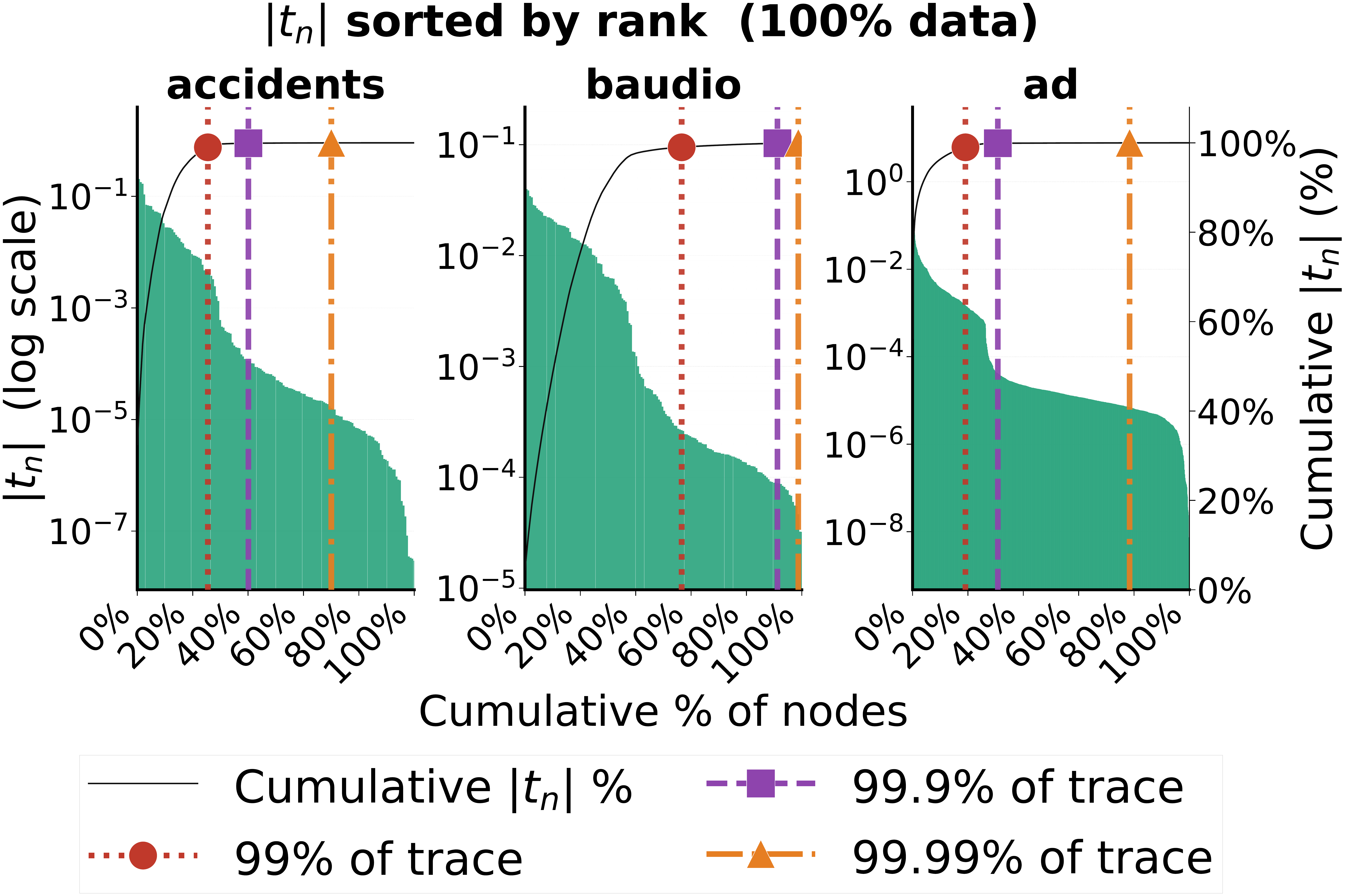}
    \caption{\textbf{Local trace is lesser concentrated.}
    Cumulative share of $\sum_n\that$ as nodes are ranked by $\that$. On the evaluated circuits, the
    local trace is distributed more broadly than the global contribution, indicating that part of the concentration in $\Th$ is
    introduced by contextual usage.}
    \label{fig:ltrace-concentration}
\end{figure}

\begin{figure}[ht]
    \centering
    \includegraphics[width=0.9\linewidth]{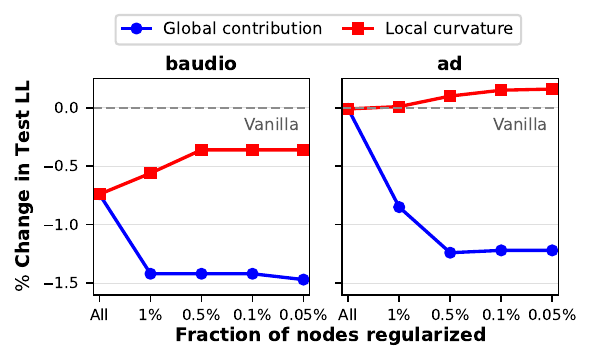}
    \caption{\textbf{Local-curvature selection preserves fit more effectively than global-contribution selection.}
    Change in test log-likelihood relative to the unregularized model as regularization is applied to all sum nodes or restricted to progressively smaller fractions ranked by global contribution or local curvature. Selecting by global contribution degrades performance whereas selecting by local curvature preserves the baseline on \texttt{baudio} and improves it on \texttt{ad}.}
    \label{fig:global-vs-local}
\end{figure}

\section{Adaptive Sharpness-Aware Learning}
\label{sec:method}

The decomposition suggests that global contribution and local curvature should play different roles during learning. The global term $\Th=\Eh[\Fn^2\tn]$ measures a node's contribution to root-level curvature, but also reflects its contextual usage. We therefore propose to use the local trace to determine \emph{where} regularization should act, while retaining the global trace penalty as the quantity being controlled. This gives a simple modification of global sharpness-aware learning: each sum node receives a gate derived from its local curvature, and the existing trace penalty is scaled by that gate.

\paragraph{Local-Curvature Gating}
\label{sec:gate}
Let $\that=\frac{1}{N}\sum_{i=1}^{N}\tn(\mathbf{x}_i)$ denote the empirical local trace of node $n$. We assign each sum node a gate $\omega_n\in[0,1]$ and define
\begin{equation*}
R_{\boldsymbol{\omega}}(\boldsymbol{\theta})
=\sum_{n\in\Sset}\omega_n\Th
=\frac{1}{N}\sum_{i=1}^{N}\sum_{n\in\Sset}\omega_n\sum_{c\in\ch(n)}
\left(\frac{\Fnc(\mathbf{x}_i)}{\theta_{nc}}\right)^2
\label{eq:weighted}
\end{equation*}
The global trace regularizer is recovered when $\omega_n=1$ for every node. We choose $\omega_n=g(\that)$ with $g$ monotone increasing, so that nodes with larger local curvature receive stronger regularization independently of their flow. To reduce sensitivity to the scale of $\that$, we use the bounded gate
\begin{equation}
\omega_n=\frac{\that}{\max_{m\in\Sset}\widehat{t}_m},
\label{eq:maxnorm}
\end{equation}
The node with the largest empirical local trace receives the full penalty, while the remaining nodes are scaled by their curvature relative to that maximum. Gates are recomputed from the current model and held fixed during each parameter update.
This separates two roles that are coupled by the global regularizer: $\that$ allocates regularization across nodes, while $\Th$ remains the curvature contribution being penalized.

\begin{table}[t]
\centering
\small
\caption{Test log-likelihood on DEBD benchmark datasets (higher is better, mean $\pm$ std over 5 seeds). Best $\mu$/method selected via validation LL. 
}
\label{tab:debd_test_ll}
\begin{tabular}{lrrr}
\toprule
\textbf{Dataset} & \textbf{Vanilla} & \textbf{Global Trace} & \makecell{\textbf{Gated}\\\textbf{Regularization}} \\
\midrule
\textit{accidents} & \underline{-26.64$\pm$0.02} & -29.79$\pm$0.13 & \textbf{-26.55$\pm$0.01} \\
\textit{ad} & \underline{-18.40$\pm$0.09} & -20.36$\pm$0.05 & \textbf{-18.10$\pm$0.04} \\
\textit{baudio} & \underline{-39.62$\pm$0.02} & -42.45$\pm$0.02 & \textbf{-39.62$\pm$0.01} \\
\textit{bbc} & -269.58$\pm$0.41 & \textbf{-256.40$\pm$0.10} & \underline{-260.07$\pm$0.26} \\
\textit{bnetflix} & \underline{-56.22$\pm$0.02} & -59.67$\pm$0.02 & \textbf{-56.21$\pm$0.02} \\
\textit{book} & \underline{-34.34$\pm$0.04} & -36.35$\pm$0.03 & \textbf{-34.16$\pm$0.01} \\
\textit{c20ng} & \textbf{-151.99$\pm$0.13} & -157.64$\pm$0.05 & \underline{-152.22$\pm$0.13} \\
\textit{cr52} & \underline{-99.09$\pm$1.15} & -102.23$\pm$1.45 & \textbf{-97.99$\pm$0.70} \\
\textit{cwebkb} & \underline{-154.71$\pm$0.33} & -157.30$\pm$0.02 & \textbf{-152.90$\pm$0.15} \\
\textit{dna} & -87.78$\pm$0.24 & \underline{-81.33$\pm$0.00} & \textbf{-81.28$\pm$0.08} \\
\textit{jester} & \textbf{-52.82$\pm$0.03} & -55.74$\pm$0.05 & \underline{-52.82$\pm$0.03} \\
\textit{kdd} & \textbf{-2.22$\pm$0.01} & -2.36$\pm$0.00 & \underline{-2.22$\pm$0.00} \\
\textit{kosarek} & \textbf{-10.59$\pm$0.01} & -11.08$\pm$0.04 & \underline{-10.59$\pm$0.02} \\
\textit{msnbc} & \underline{-6.12$\pm$0.01} & -6.47$\pm$0.01 & \textbf{-6.08$\pm$0.00} \\
\textit{msweb} & \textbf{-9.73$\pm$0.01} & -10.40$\pm$0.02 & \underline{-9.73$\pm$0.01} \\
\textit{nltcs} & \underline{-6.00$\pm$0.00} & -6.48$\pm$0.00 & \textbf{-6.00$\pm$0.00} \\
\textit{plants} & \underline{-12.68$\pm$0.03} & -15.99$\pm$0.02 & \textbf{-12.67$\pm$0.01} \\
\textit{pumsb\_star} & \underline{-22.78$\pm$0.06} & -28.29$\pm$0.06 & \textbf{-22.75$\pm$0.04} \\
\textit{tmovie} & \textbf{-42.16$\pm$0.03} & -48.84$\pm$0.04 & \underline{-42.21$\pm$0.06} \\
\textit{tretail} & \textbf{-10.84$\pm$0.01} & -10.98$\pm$0.01 & \underline{-10.85$\pm$0.01} \\
\bottomrule
\end{tabular}
\end{table}

\paragraph{Gated EM Update}
\label{sec:emupdate-method}
Let $\Nnc=\sum_i\Fnc(\mathbf{x}_i)$ and $\Snc=\sum_i\Fnc(\mathbf{x}_i)^2$ denote the expected edge count and edge-flow second moment.
With fixed gates, directly optimizing the empirical trace penalty gives the following objective:
\begin{equation*}
\max_{\boldsymbol{\theta}_n}\ \sum_c\Nnc\log\theta_{nc}
-\mu\omega_n\sum_c\frac{\Snc}{\theta_{nc}^2}
\quad\text{s.t.}\quad \sum_c\theta_{nc}=1.
\label{eq:emobj}
\end{equation*}
Its stationarity condition, $\Nnc/\theta_{nc}-\lambda+2\mu\omega_n\Snc/\theta_{nc}^3=0$, is cubic in $\theta_{nc}$.
To preserve the tractable EM updates of global trace-regularized learning, we apply the same surrogate as used by \cite{SureshAAAI2026} with the node-specific strength
$\mu_n=\mu\omega_n$, which yields the following update equation.


\begin{proposition}[]
\label{prop:gated-update}
For fixed gates $\{\omega_n\}_{n\in\Sset}$, under the surrogate used by global trace-regularized EM, the update for each outgoing weight of sum node $n$ satisfies
$
\lambda_n\theta_{nc}^2
-
\Nnc\theta_{nc}
-
\mu\omega_n\Nnc
=
0,
$
and its positive solution is given by
\begin{equation}
\theta_{nc}
=
\frac{
\Nnc+
\sqrt{\Nnc^2+4\lambda_n\mu\omega_n\Nnc}
}{
2\lambda_n
},
\label{eq:gatedupdate}
\end{equation}
\end{proposition}

The update thus preserves the form and asymptotic complexity of the global method, changing only the effective node-wise strength,
$\mu\mapsto\mu\omega_n$. It recovers global trace regularization when $\omega_n=1$ and approaches the unregularized EM update as $\omega_n\rightarrow0$. Though we adopt this simple gate as proof of concept, Prop.~\ref{prop:gated-update} holds for any monotone gate, and richer gate functions are natural directions for future work.

\section{Experiments}
\label{sec:experiments}

We organize the experiments around three questions that connect the theory to the proposed method:\\
\noindent\textbf{(Q1)} How strongly does contextual flow shape the distribution of global trace contributions?\\
\noindent\textbf{(Q2)} Does selecting nodes by local curvature preserve model fit more effectively than selecting by global contribution?\\
\noindent\textbf{(Q3)} Does global trace regularization underfit, and can adaptive local-curvature gating recover the lost capacity?

\paragraph{Experimental setup.}
We evaluate on the 20 binary density-estimation benchmarks (DEBD) using Hidden Chow-Liu Trees (HCLTs)~\cite{LiuNeurIPS2021}. All models are implemented in PyJuice~\cite{LiuICML2024} with a latent size $100$ and trained using EM. We compare the unregularized model, global trace regularization, and the proposed local-curvature-gated regularizer.
Results are averaged over $5$ random seeds, and the regularization strength $\mu$ is selected independently for each method using validation log-likelihood.
Further experimental details are provided in the appendix.

\paragraph{Q1: Anatomy of the global trace.}
We first examine how the factors in $\Tnn=\Fn^2\tn$ shape global sharpness. Figure~\ref{fig:depth-contrib} shows that global trace contributions concentrate in earlier circuit partitions, consistent with attenuation theory through upstream sum-node routing. Figure~\ref{fig:ltrace-concentration} further shows that local curvature is substantially less concentrated than global contribution: the top $10\%$ of nodes account for more than $99.99\%$ of $\Th$, but it takes over $60\%$ of nodes to contribute the same for $\that$. Thus, concentration in the global trace is primarily caused by the contextual usage than local curvature.

\paragraph{Q2: Controlled node selection.}
We next isolate the effect of the ranking criterion. Holding the regularization strength and selected fraction fixed, we apply the trace penalty to nodes chosen by global contribution $\Th$ or local curvature $\that$. 
Figure~\ref{fig:global-vs-local} compares the two criteria on \texttt{baudio} and \texttt{ad}. Restricting the penalty to progressively smaller sets of high-$\Th$ nodes reduces test log-likelihood on both datasets. In contrast, selecting by $\that$ preserves the unregularized fit on \texttt{baudio} and improves it slightly on \texttt{ad}. These results show that global contribution and local curvature provide materially different signals for allocating regularization.

\paragraph{Q3: Underfitting and recovery.}
Table~\ref{tab:debd_test_ll} compares the three methods across all 20 DEBD datasets. Global trace regularization improves over the unregularized model on only 2 datasets and degrades performance on the remaining 18, indicating
underfitting. In contrast, the proposed gated method outperforms global trace regularization and matches or improves upon the unregularized model on majority of the datasets.
These results show that allocating the trace penalty according to local curvature substantially mitigates the loss of fit induced by uniform global regularization.

\section{Conclusion}
\label{sec:conclusion}
Overall, in this paper we studied sharpness aware learning in probabilistic circuits through a compositional lens. Our analysis showed that a node's global curvature contribution separates exactly into contextual usage and intrinsic local curvature, providing insights into why global trace regularization can flatten the wrong parts of the model and induce underfitting. Guided by this decomposition, we introduced a gated regularizer that preserves tractable learning, while allocating regularization effectively to prevent underfitting. Future work involves developing richer gating functions that jointly account for local geometry and contextual influence. More broadly, the decomposition also opens up directions for  curvature-aware model compression, targeted robustness interventions, and adaptive circuit design, where contextual usage and local sensitivity can jointly guide which components to preserve, regularize, prune, or expand.

\section*{Acknowledgments}
SN and SS gratefully acknowledge the generous support by the AFOSR award FA9550-23-1-0239, the ARO award W911NF2010224 and the DARPA Assured Neuro Symbolic Learning and Reasoning (ANSR) award HR001122S0039. CK and HS gratefully acknowledge Dr. Anji Liu for the discussions related to the work and CK, HS, SPV and YZ thank IIT Palakkad for the access to Madhava Cluster. 

\bibliography{refer}
\appendix

\onecolumn

\noindent \hrulefill
\section*{Supplementary Material:\\A Compositional Theory of Curvature in Probabilistic Circuits}

\noindent \hrulefill

\section{Preliminaries and Notation}
\label{app:notation}

We first revisit the definitions and notation used throughout the proofs, so that each subsequent argument can be read in isolation. 

\paragraph{Probabilistic circuits.}
We consider smooth and decomposable probabilistic circuits with alternating layers of sum and product nodes over a set of tractable input distributions. The root node $r$ computes the model density $\pr(\x)$. Each sum node $n\in\Sset$ computes a convex combination of its children, $p_n(\x)=\sum_{c\in\ch(n)}\theta_{nc}\,p_c(\x)$, with nonnegative weights normalized on the simplex, $\sum_{c\in\ch(n)}\theta_{nc}=1$. We write $\ell(\boldsymbol{\theta};\x)=-\log \pr(\x)$ for the per-sample negative log-likelihood (NLL) and study the trace of its Hessian with respect to the sum weights $\boldsymbol{\theta}$.

\paragraph{Output ratios and routing responsibilities.}
For a sum node $n$ and child $c$ we define the \emph{output ratio} $\rho_{nc}(\x)=p_c(\x)/p_n(\x)$ and the \emph{routing responsibility} $r_{nc}(\x)=\theta_{nc}\,\rho_{nc}(\x)$. The responsibilities form a distribution over children, $\sum_{c\in\ch(n)}r_{nc}(\x)=1$, and we collect the output ratios into the vector $\rr_n(\x)=(\rho_{nc}(\x))_{c\in\ch(n)}$.

\paragraph{Circuit flow.}
The circuit flow measures the top-down usage of each node, and is defined by the node-type recursion below.
\begin{definition}[Circuit flow]\label{def:app-circuitflow}
The flow $F_n(\x)$ of a node $n$ is
\begin{equation}
F_n(\x)=
\begin{cases}
1, & n=r \ \text{(root)},\\[4pt]
\displaystyle\sum_{m\in\pa(n)} F_m(\x), & n \ \text{a sum node},\\[10pt]
\displaystyle\sum_{m\in\pa(n)} \theta_{mn}\,\frac{p_n(\x)}{p_m(\x)}\,F_m(\x), & n \ \text{a product or input node},
\end{cases}
\label{eq:app-flowrec}
\end{equation}
\end{definition}
where $\pa(n)$ denotes the parents of $n$.
By the alternating-layer property, the responsibility factor $\theta_{mn}\,p_n/p_m=r_{mn}$ attaches on exactly the \emph{sum edges} of the circuit (edges out of a sum node), while product and input nodes inherit their parents' flow undivided.

\paragraph{Edge flow.}
The flow carried by the edge from sum node $n$ to child $c$ is $F_{nc}(\x)=\theta_{nc}\,\rho_{nc}(\x)\,F_n(\x)=r_{nc}(\x)\,F_n(\x)$. It satisfies $F_{nc}/\theta_{nc}=\rho_{nc}\,F_n$.

\paragraph{Local and global trace.}
The \emph{local trace} of sum node $n$ is $\tn(\x)=\sum_{c\in\ch(n)}\rho_{nc}(\x)^2$, and its \emph{global trace contribution} is $\Tnn(\x)=\sum_{c\in\ch(n)}(F_{nc}(\x)/\theta_{nc})^2$. Theorem~\ref{thm:decomp} shows these are related by the exact identity $\Tnn=\Fn^2\,\tn$. Summing over sum nodes recovers the full NLL Hessian trace, $\Tr(\nabla^2_{\boldsymbol{\theta}}\ell)=\sum_{n\in\Sset}\Tnn\ge 0$. Empirical batch estimators are denoted $\that$ and $\Th$.


\section{Proofs for Theoretical Results}
In this section, we provide complete, self-contained proofs for all of the key lemmas, propositions, and theorems stated in the
main paper
\begin{theorem}[]
Consider any sum node \(n\) with strictly positive outgoing weights in a smooth, decomposable PC. For any input \(\mathbf{x}\), its global trace contribution $\Tnn(\mathbf{x})$
factorizes as
$
\Tnn(\mathbf{x})
=
\Fn(\mathbf{x})^2\,\tn(\mathbf{x}),
$
$
\tn(\mathbf{x})
=
\sum_{c\in\ch(n)}
(p_c(\mathbf{x})/p_n(\mathbf{x}))^2
$
Consequently,
$
\Tr\left(
\nabla_{\boldsymbol{\theta}}^2
\ell(\boldsymbol{\theta};\mathbf{x})
\right)
=
\sum_{n\in\Sset}
\Fn(\mathbf{x})^2\,\tn(\mathbf{x}).
$
\end{theorem}



\begin{proof}
The global trace contribution of node $n$ is the sum over its outgoing edges of the squared diagonal Hessian entries, which equal $(F_{nc}(\x)/\theta_{nc})^2$. Substituting the edge-flow identity $F_{nc}/\theta_{nc}=\rho_{nc}\,F_n$ from Section~\ref{app:notation},
\begin{equation}
    \Tnn(\mathbf{x})
    =
    \sum_{c \in \ch(n)}
    \left(\frac{F_{nc}(\mathbf{x})}{\theta_{nc}}\right)^2
    =
    \sum_{c \in \ch(n)}
    \left(
            \frac{p_c(\textbf{x})}{p_n(\textbf{x})}\,F_n(\textbf{x})
    \right)^2 .
\end{equation}
The flow $F_n(\x)$ does not depend on the child index $c$, so it factors out of the sum,
\begin{equation}
    \Tnn(\mathbf{x})
    = F_n(\textbf{x})^2 \ \sum_{c \in \ch(n)}
    \left(
            \frac{p_c(\textbf{x})}{p_n(\textbf{x})}
    \right)^2 .
\end{equation}
Recognizing the remaining sum as the local trace $\tn(\x)=\sum_{c\in\ch(n)}\rho_{nc}(\x)^2$ yields the factorization,
\begin{equation}
        \Tnn(\mathbf{x})
    = F_n(\textbf{x})^2 \ \tn(\mathbf{x}) .
\end{equation}
Summing over all sum nodes gives the full Hessian trace $\Tr(\nabla^2_{\boldsymbol{\theta}}\ell)=\sum_{n\in\Sset}F_n^2\,\tn\ge 0$, since each term is a product of squares.
\end{proof}

\begin{proposition}[Rank-one local Hessian]
The gradient and Hessian of $\ell_n$ in the edge weights 
$\boldsymbol{\theta}_n$ are
$\nabla_{\boldsymbol{\theta}_n}\ell_n = -\rr_n$ and
$\Hn(\mathbf{x}) = \nabla^2_{\boldsymbol{\theta}_n}\ell_n = \rr_n\rr_n^\top$. Hence, $\Hn$ is
positive semidefinite and, whenever $\rr_n\neq 0$, has rank one, with unique nonzero eigenvalue
$\lambda_{\max}(\Hn)=\lVert\rr_n\rVert_2^2=\tn$. Consequently
$\tn = \Tr(\Hn) = \lVert\Hn\rVert_2 = \lVert\Hn\rVert_F$.
\end{proposition}

\begin{proof}
    $\ell_n(\boldsymbol{\theta}_n;\mathbf{x})=-\log\big(\sum_{c}\theta_{nc}p_c(\mathbf{x})\big)$,\ $\rho_{nc}(\mathbf{x})=p_c(\mathbf{x})/p_n(\mathbf{x})$,
    $p_n(\mathbf{x}) = \sum_{c}\theta_{nc}p_c(\mathbf{x})$\\
    First order derivative
    \begin{equation}
        \frac{\partial \ell_n(\boldsymbol{\theta}_n;\mathbf{x})}{\partial \theta_{nc}} = \frac{-\partial \log\big(\sum_{c}\theta_{nc}p_c(\mathbf{x})\big)}{\partial \theta_{nc}}
        = \frac{-1}{p_n(\mathbf{x})} \frac{\partial \sum_{c}\theta_{nc}p_c(\mathbf{x})}{\partial \theta_{nc}}
    \end{equation}
    \begin{equation}
         \frac{\partial \ell_n(\boldsymbol{\theta}_n;\mathbf{x})}{\partial \theta_{nc}} = -\left(\frac{p_c(\mathbf{x})}{p_n(\mathbf{x})}\right)
    \end{equation}
    \begin{equation}
        \implies \nabla_{\boldsymbol{\theta}_n}\ell_n = -\rr_n
    \end{equation}

Second order derivative
\begin{equation}
    \frac{\partial^2 \ell_n(\boldsymbol{\theta}_n;\mathbf{x})}{\partial \theta_{nc'} \theta_{nc}} = -\frac{\partial \left(p_c(\mathbf{x})/p_n(\mathbf{x})\right)}{\partial \theta_{nc'}} = -\left[ \frac{p_n(\mathbf{x}) \ p_c(\mathbf{x})' - p_n(\mathbf{x})' \ p_c(\mathbf{x})}{p_n(\mathbf{x})^2}\right]
\end{equation}
As $p_c(\mathbf{x})$ is independent of $\theta_{nc'}$ we get
\begin{equation}
    \frac{\partial \left(p_c(\mathbf{x})/p_n(\mathbf{x})\right)}{\partial \theta_{nc'}} = -\left[ \frac{ - p_{c'}(\mathbf{x}) \ p_c(\mathbf{x})}{p_n(\mathbf{x})^2}\right] =  \frac{  p_{c'}(\mathbf{x})} {p_n(\mathbf{x})} \ \frac{p_c(\mathbf{x})}{p_n(\mathbf{x})} =\rho_{nc'} \ \rho_{nc}
   \end{equation}

Thus $\Hn=\rr_n\rr_n^\top$. A nonzero outer product has rank one and unique nonzero eigenvalue
$\|\rr_n\|_2^2$. For a rank-one positive semidefinite matrix, the trace, spectral norm, and
Frobenius norm all equal this eigenvalue.
Trace of local Hessian
\begin{equation}
     \Tr(\Hn(\mathbf{x})) = \Tr(\rr_n\rr_n^\top) = \sum_c \left(\frac{p_c(\mathbf{x})}{p_n(\mathbf{x})}\right)^2 = \tn(\mathbf{x})
\end{equation}
\end{proof}

\begin{proposition}

For any two sum nodes $i,j$, $T_i>T_j \iff t_i/t_j > (F_j/F_i)^2$. In particular, a locally less curved node $i$ with $t_i<t_j$ outranks a locally sharper node $j$ whenever $F_i/F_j > \sqrt{t_j/t_i}$.
\end{proposition}

\begin{proof}
    $$T_i = F_i^2 \ t_i, T_j = F_j^2 \ t_j $$
    $$ T_i > T_j \iff F_i^2 \ t_i > F_j^2 \ t_j $$
    $$ T_i > T_j  \iff \frac{t_i}{t_j} > \left(\frac{F_j}{F_i}\right)^2$$
\end{proof}

\begin{lemma}

For a tree-structured PC, the node flow is the product of routing responsibilities along the
unique root-to-$n$ path $\pi(r,n)$, taken over its \emph{sum} edges ($E_{sum}$) only:
$\Fn(\mathbf{x}) = \prod_{e\in\pi(r,n)\cap E_{\mathrm{sum}}} r_e(\mathbf{x})$, with
$r_{(m,n')}=\theta_{mn'}\,p_{n'}/p_m$. For a DAG,
$\Fn(\mathbf{x})=\sum_{\pi\in\Pi(r,n)}\prod_{e\in\pi\cap E_{\mathrm{sum}}} r_e(\mathbf{x})$,
summing over all root-to-$n$ paths $\Pi(r,n)$.
\end{lemma}
\begin{proof}
We proceed by induction on the depth $d$ of node $n$ in the tree-structured PC.
 
\paragraph{Base case ($d=0$).}
Node $n$ is the root $r$.  By definition, $F_r(\mathbf{x}) = 1$.
The product over an empty index set (no edges on the root-to-root path) is also $1$,
so the claim holds trivially. 
 
\paragraph{Inductive step.}
Assume the claim holds for every node at depth strictly less than $d$.
Let $n$ be a node at depth $d$, and let $m = \mathrm{pa}(n)$ denote its unique parent
( PC is tree-structured).
We consider two cases according to the type of the edge $(m,n)$.

\paragraph{Case 1: $(m,n)$ is a product edge} (i.e.\ $m$ is a product node).
 
By the flow recursion, product nodes propagate their flow without weighting:
\begin{equation*}
    F_n(\mathbf{x}) \;=\; F_m(\mathbf{x}).
\end{equation*}
Since $m$ is at depth $d-1$, the inductive hypothesis gives
\begin{equation*}
    F_m(\mathbf{x}) \;=\;
    \prod_{e\,\in\,\pi(r,m)\,\cap\, E_{\mathrm{sum}}} r_e(\mathbf{x}).
\end{equation*}
Because $(m,n)\notin E_{\mathrm{sum}}$,
the sum-edge sets along the two paths coincide:
$\pi(r,n)\cap E_{\mathrm{sum}} = \pi(r,m)\cap E_{\mathrm{sum}}$.
Therefore
\begin{equation}
    F_n(\mathbf{x}) \ = \
    \prod_{e\,\in\,\pi(r,n)\,\cap\, E_{\mathrm{sum}}} r_e(\mathbf{x}).
\end{equation}

\noindent\textbf{Case 2: $(m,n)$ is a sum edge} (i.e.\ $m$ is a sum node).
 
By the flow recursion and the definition of routing responsibility
$r_{(m,n)}(\mathbf{x}) = \theta_{mn}\,{p_n(\mathbf{x})}/{p_m(\mathbf{x})}$:
\begin{equation*}
    F_n(\mathbf{x})
    \;=\;
    \theta_{mn}\,\frac{p_n(\mathbf{x})}{p_m(\mathbf{x})}\,F_m(\mathbf{x})
    \;=\;
    r_{(m,n)}(\mathbf{x})\,F_m(\mathbf{x}).
\end{equation*}
Since $m$ is at depth $d-1$, the inductive hypothesis gives
\begin{equation*}
    F_m(\mathbf{x})
    \;=\;
    \prod_{e\,\in\,\pi(r,m)\,\cap\, E_{\mathrm{sum}}} r_e(\mathbf{x}).
\end{equation*}
Because $(m,n)\in E_{\mathrm{sum}}$ and $\pi(r,n) = \pi(r,m)\cup\{(m,n)\}$,
 the new sum-edge factor yields
\begin{align*}
    F_n(\mathbf{x})
    &\;=\;
    r_{(m,n)}(\mathbf{x})
    \cdot
    \prod_{e\,\in\,\pi(r,m)\,\cap\, E_{\mathrm{sum}}} r_e(\mathbf{x}) \\[4pt]
    &\;=\;
    \prod_{e\,\in\,\pi(r,n)\,\cap\, E_{\mathrm{sum}}} r_e(\mathbf{x}).
\end{align*}

Both cases establish the inductive step, so by strong induction the formula
\begin{equation*}
    F_n(\mathbf{x})
    \;=\;
    \prod_{e\,\in\,\pi(r,n)\,\cap\, E_{\mathrm{sum}}} r_e(\mathbf{x})
\end{equation*}
holds for every node $n$ in the tree-structured PC.

For general DAG structure

\textbf{Base case ($d = 0$).}\enspace
Node $n$ is the root $r$.
By definition $F_r(\mathbf{x}) = 1$.
The only root-to-root path is the empty path $\pi_\varnothing$;
its edge product equals $1$ (empty product), and the sum over
this single path equals $1$. Hence the claim holds. 
 
\medskip
\noindent\textbf{Inductive step.}\enspace
Assume the claim holds for every node at depth strictly less than
$d \geq 1$.
Let $n$ be a node at depth $d$ with parent set
$\mathrm{pa}(n)=\{m_1,\ldots,m_k\}$, each parent at depth $d-1$.
By the inductive hypothesis, for every $m_i\in\mathrm{pa}(n)$:
\begin{equation}
    F_{m_i}(\mathbf{x})
    \;=\;
    \sum_{\pi\,\in\,\Pi(r,\,m_i)}
    \prod_{e\,\in\,\pi\cap E_{\mathrm{sum}}} r_e(\mathbf{x}).
    \label{eq:IH}
\end{equation}
We consider two cases according to the type of $n$.
 
\medskip
\noindent\textbf{Case~1: $n$ is a sum node.}
 
By the alternating-layer property every parent $m_i$ is a product node,
so every edge $(m_i,n)\notin E_{\mathrm{sum}}$.
The flow recursion for a sum node gives
\begin{equation*}
    F_n(\mathbf{x})
    \;=\;
    \sum_{m_i\,\in\,\mathrm{pa}(n)} F_{m_i}(\mathbf{x}).
\end{equation*}
Substituting~\eqref{eq:IH}:
\begin{equation*}
    F_n(\mathbf{x})
    \;=\;
    \sum_{m_i\,\in\,\mathrm{pa}(n)}
    \sum_{\pi\,\in\,\Pi(r,\,m_i)}
    \prod_{e\,\in\,\pi\cap E_{\mathrm{sum}}} r_e(\mathbf{x}).
\end{equation*}
Every root-to-$n$ path $\pi\in\Pi(r,n)$ passes through exactly one
parent $m_i\in\mathrm{pa}(n)$ and then traverses the product edge
$(m_i,n)$.
Because $(m_i,n)\notin E_{\mathrm{sum}}$, adding this edge does not
change the set of sum edges on the path:
\begin{equation*}
    \pi\cap E_{\mathrm{sum}}
    \;=\;
    \pi'\cap E_{\mathrm{sum}},
    \qquad
    \text{where }
    \pi'=\pi(r,m_i)
    \text{ is the prefix of }\pi\text{ up to }m_i.
\end{equation*}
Hence the double sum exactly enumerates all root-to-$n$ paths:
\begin{equation*}
    F_n(\mathbf{x})
    \;=\;
    \sum_{\pi\,\in\,\Pi(r,n)}
    \prod_{e\,\in\,\pi\cap E_{\mathrm{sum}}} r_e(\mathbf{x}).
\end{equation*}
 
\medskip
\noindent\textbf{Case~2: $n$ is a product or input node.}
 
By the alternating-layer property every parent $m_i$ is a sum node,
so every edge $(m_i,n)\in E_{\mathrm{sum}}$ carries routing responsibility
$r_{(m_i,n)}(\mathbf{x})=\theta_{m_in}\,p_n(\mathbf{x})/p_{m_i}(\mathbf{x})$.
The flow recursion reads
\begin{equation*}
    F_n(\mathbf{x})
    \;=\;
    \sum_{m_i\,\in\,\mathrm{pa}(n)}
    r_{(m_i,n)}(\mathbf{x})\,F_{m_i}(\mathbf{x}).
\end{equation*}
Substituting~\eqref{eq:IH}:
\begin{equation*}
    F_n(\mathbf{x})
    \;=\;
    \sum_{m_i\,\in\,\mathrm{pa}(n)}
    r_{(m_i,n)}(\mathbf{x})
    \sum_{\pi'\,\in\,\Pi(r,\,m_i)}
    \prod_{e\,\in\,\pi'\cap E_{\mathrm{sum}}} r_e(\mathbf{x}).
\end{equation*}
Every root-to-$n$ path $\pi\in\Pi(r,n)$ decomposes uniquely as
a prefix $\pi'\in\Pi(r,m_i)$ for exactly one $m_i\in\mathrm{pa}(n)$
followed by the sum edge $(m_i,n)$.
Since $(m_i,n)\in E_{\mathrm{sum}}$, appending this edge extends the
edge product by exactly one factor:
\begin{equation*}
    \prod_{e\,\in\,\pi\cap E_{\mathrm{sum}}} r_e(\mathbf{x})
    \;=\;
    r_{(m_i,n)}(\mathbf{x})
    \cdot
    \prod_{e\,\in\,\pi'\cap E_{\mathrm{sum}}} r_e(\mathbf{x}).
\end{equation*}
Re-indexing the double sum over all root-to-$n$ paths therefore gives
\begin{align*}
    F_n(\mathbf{x})
    &\;=\;
    \sum_{m_i\,\in\,\mathrm{pa}(n)}
    \;\sum_{\pi'\,\in\,\Pi(r,\,m_i)}
    r_{(m_i,n)}(\mathbf{x})
    \prod_{e\,\in\,\pi'\cap E_{\mathrm{sum}}} r_e(\mathbf{x}) \\[6pt]
    &\;=\;
    \sum_{\pi\,\in\,\Pi(r,n)}
    \prod_{e\,\in\,\pi\cap E_{\mathrm{sum}}} r_e(\mathbf{x}).
\end{align*}
 
\medskip
\noindent\textbf{Conclusion.}\enspace
Both cases establish the inductive step.
By strong induction,
\begin{equation*}
    F_n(\mathbf{x})
    \;=\;
    \sum_{\pi\,\in\,\Pi(r,n)}
    \prod_{e\,\in\,\pi\cap E_{\mathrm{sum}}} r_e(\mathbf{x})
\end{equation*}
\end{proof}

\begin{corollary}

Consider a tree-structured PC. If every sum-edge responsibility on the root-to-$n$ path satisfies $r_e(\mathbf{x})\le\rho<1$, then
$
\Fn(\mathbf{x})\le\rho^{d_\Sigma(n)},
$
and
$
\Tnn(\mathbf{x})\le\rho^{2d_\Sigma(n)}\tn(\mathbf{x}),
$
where $d_\Sigma(n)$ is the number of upstream sum edges.
\end{corollary}
\begin{proof}

\medskip
\noindent\textbf{Part 1: Bound on $F_n(\mathbf{x})$.}
 
By the Lemma \ref{lemma:propagation}, the node flow at $n$
in a tree-structured PC is the product of routing responsibilities over
all sum edges on the unique root-to-$n$ path:
\[
  F_n(\mathbf{x})
  \;=\;
  \prod_{e\,\in\,\pi(r,n)\,\cap\,E_{\mathrm{sum}}} r_e(\mathbf{x}).
\]
This product contains exactly $d_\Sigma(n)$ factors.
By hypothesis, each factor satisfies $r_e(\mathbf{x})\le\rho<1$, so:
\[
  F_n(\mathbf{x})
  \;=\;
  \prod_{e\,\in\,\pi(r,n)\,\cap\,E_{\mathrm{sum}}} r_e(\mathbf{x})
  \;\le\;
  \rho^{\,d_\Sigma(n)}.
\]

\medskip
\noindent\textbf{Part 2: Bound on $T_n(\mathbf{x})$.}
 
By the decomposition in Theorem \ref{thm:decomp}:
\[
  T_n(\mathbf{x}) \;=\; F_n(\mathbf{x})^2 \cdot t_n(\mathbf{x}).
\]
Since $t_n(\mathbf{x})\ge 0$ (it is a sum of squares), we may apply the 
flow bound directly:
\[
  T_n(\mathbf{x})
  \;=\;
  F_n(\mathbf{x})^2 \cdot t_n(\mathbf{x})
  \;\le\;
  \left(\rho^{\,d_\Sigma(n)}\right)^{\!2}\cdot t_n(\mathbf{x})
  \;=\;
  \rho^{\,2d_\Sigma(n)}\,t_n(\mathbf{x}).
\]
\end{proof}
 
\subsection{Adaptive Gated Regularization}
\begin{proposition}[]

For fixed gates $\{\omega_n\}_{n\in\Sset}$, under the surrogate used by global trace-regularized EM, the update for each outgoing weight of sum node $n$ satisfies
$
\lambda_n\theta_{nc}^2
-
\Nnc\theta_{nc}
-
\mu\omega_n\Nnc
=
0,
$
and its positive solution is given by
\begin{equation}
\theta_{nc}
=
\frac{
\Nnc+
\sqrt{\Nnc^2+4\lambda_n\mu\omega_n\Nnc}
}{
2\lambda_n
},
\end{equation}
\end{proposition}
\begin{proof}
    Adding the adaptive regularizer to the EM objective of the PC we obtain

\[
\begin{aligned}
\theta^*_{n\cdot} = \arg\max_{\theta_{nc}} \quad & \sum\limits_{c \in \mathrm{ch}(n) }\Nnc(x)\: \log \,\theta_{nc} \\
\text{subject to} \quad & \sum\limits_{c \in \mathrm{ch}(n)} \theta_{nc} = 1, \\
                        & \omega_n(x) \sum\limits_{c \in ch(n)}  \left(\dfrac{\Nnc(x)}{\theta_{nc}}\right) \leq m
\end{aligned}
\]
 The Lagrangian formulation of the constrained maximization objective is 
\[\begin{aligned} \mathcal{L}(\theta_{n\cdot}, \lambda, \mu) = &\sum\limits_{c \in \mathrm{ch}(n) }\Nnc(x)\: \log \,\theta_{nc}  - \lambda \left( \sum\limits_{c \in \mathrm{ch}(n)} \theta_{nc} - 1\right) \\ & - \mu \left( \omega_n(x)\sum\limits_{c \in \mathrm{ch}(n)} \left(\dfrac{\Nnc(x)}{\theta_{nc}} \right)^2 - m \right)\end{aligned}\]  

\noindent Differentiating wrt $\theta_{nc}$ we get,
    $$\dfrac{\partial\mathcal{L}(\theta_{n\cdot}, \lambda, \mu)}{\partial \theta_{nc}} = \dfrac{\Nnc}{\theta_{nc}} - \lambda +\mu \left( \omega_n \dfrac{\Nnc}{\theta_{nc}^2}\right)$$
    equating this to $0$ yields the quadratic equation $$
   \lambda \ \theta_{nc}^2 - \Nnc \ \theta_{nc} - \mu \ \omega_n \Nnc = 0$$
    whose solution obtained by the quadratic formula is $$\theta_{nc} = \dfrac{\Nnc + \sqrt{\Nnc(x)^2 + 4 \lambda \mu \ \omega_n \Nnc}}{2 \lambda}$$

\end{proof}
\begin{algorithm}[ht]
\caption{\textbf{Gated Local Trace EM for Probabilistic Circuits}}
\label{alg:gated-local-trace-em}
\caption{\textbf{Gated Marginals EM for Probabilistic Circuits}}
\label{alg:gated-marginals-em}
\begin{algorithmic}[1]
 
\Require
  Probabilistic circuit $\PC$ with sum-node parameters
  $P=\{\theta_{nc}\}_{(n,c)\in E}$;
  dataset $D=\{x^{(i)}\}_{i=1}^{N}$;
  regularization weight $\mu\geq 0$;
  simplex constraint weight $\lambda>0$;
  smoothing factor $\alpha\in(0,1]$;
  monotone gate function $g:\R_{\geq 0}\to[0,1]$;
  number of epochs $E$
 
\Ensure
  Updated sum-node parameters $P$ minimising the gated
  marginals regularised log-likelihood
 
\State \textbf{Initialise:}
  for each sum-node $n$, set $\theta_{n\cdot}$ uniformly on its simplex
\State Set $\omega_n \leftarrow g(0)$ for all sum-nodes $n$
  \algorithmiccomment{gate weights; recomputed after first epoch}
 
\For{epoch $= 1,\ldots,E$}
  {\small\hfill\textcolor{algcomment}{// repeat until convergence or max epochs}}
 
  \Statex \quad\emph{\textbf{E-step:} Compute expected edge flows and node marginals}
  \State Run forward--backward passes on $\PC$ over $D$
  \State Obtain edge flows $\bigl\{F_{nc}(x)\bigr\}_{(n,c)\in E,\;x\in D}$
    \algorithmiccomment{$O\!\left(|P|\,|D|\right)$ time}
  \State Obtain node marginals $\bigl\{p_n(x),\,p_c(x)\bigr\}_{n\in V,\;x\in D}$
    \algorithmiccomment{available from the forward pass at no extra cost}
 
  \Statex \quad\emph{\textbf{Local Trace Computation (marginal ratio):}}
  \ForAll{sum-nodes $n$}
    \State Aggregate flows:
      $\;\bar{F}_{nc} \leftarrow \displaystyle\sum_{x\in D} F_{nc}(x)$
    \State Compute dataset-averaged marginal ratio for each child $c$:
    \[
      \bar{r}_{nc}
      \;\leftarrow\;
      \frac{1}{|D|}\sum_{x\in D}
      \frac{p_c(x)}{p_n(x)}
    \]
    \State Compute local trace via marginals:
    \[
      \hat{t}_{n}
      \;\leftarrow\;
      \sum_{c\,\in\,\ch(n)}
      \bar{r}_{nc}^{\;2}
      \;=\;
      \sum_{c\,\in\,\ch(n)}
      \left(
        \frac{1}{|D|}\sum_{x\in D}
        \frac{p_c(x)}{p_n(x)}
      \right)^{\!2}
    \]
    \algorithmiccomment{depth-agnostic: no $F_n$ factor; no $\theta_{nc}$ dependence}
  \EndFor
 
  \Statex \quad\emph{\textbf{Gate Weight Update:}}
  \ForAll{sum-nodes $n$}
    \State $\omega_{n} \leftarrow g\!\left(\hat{t}_{n}\right)$
      \algorithmiccomment{e.g.\ $g(\hat{t}_n) = \hat{t}_n/\max_{n'}\hat{t}_{n'}$\;
             (Proposed);}
  \EndFor
 
  \Statex \quad\emph{\textbf{M-step:} Gated sharpness-aware parameter update}
  \ForAll{sum-nodes $n$}
    \algorithmiccomment{updates are independent per node}
    \ForAll{child edges $(n\!\to\!c)$}
      \State
      \(
        \tilde{\theta}_{nc}
        \;\leftarrow\;
        \dfrac{
          \bar{F}_{nc}
          +
          \sqrt{
            \bar{F}_{nc}^{\;2}
            +
            4\,\lambda\,(\mu\,\omega_{n})\,\bar{F}_{nc}
          }
        }{2\,\lambda}
      \)
      \algorithmiccomment{same closed form as global trace; effective strength $\mu_n = \mu\,\omega_n$}
    \EndFor
    \State Normalise:
      $\;\tilde{\theta}_{nc}
        \leftarrow
        \tilde{\theta}_{nc}
        \big/
        \displaystyle\sum_{c'\in\ch(n)}\tilde{\theta}_{nc'}
      $\quad$\forall\, c\in\ch(n)$
      \algorithmiccomment{project onto probability simplex}
  \EndFor
 
  \State
    $\theta \leftarrow (1-\alpha)\,\theta + \alpha\,\tilde{\theta}$
    \algorithmiccomment{running-average smoothing for noisy mini-batch flows}
 
\EndFor
 
\State \Return $P \equiv \theta$
 
\end{algorithmic}
\end{algorithm}

\section{Experimental Setup and Implementation Details}
In this section, we provide the details regarding the datasets, model architecture, hyperparameter used for the experimental results reported in the main paper.

\begin{table}[ht]
\centering
\caption{Overview of the $20$ binary density estimation datasets, showing the number of variables and the number of instances in the training, validation, and test splits.}
\label{tab:20DEBD}
\begin{tabular}{lrrrr}
\toprule
\textbf{Dataset Name}       & \textbf{\#vars}     & \textbf{\#train}       & \textbf{\#valid }      & \textbf{\#test}        \\
\midrule
nltcs         & $16$       & $16181$       & $2157$        & $3236$        \\
msnbc         & $17$       & $291326$      & $38843$       & $58265$       \\
kdd           & $65$       & $180092$      & $19907$       & $34955$       \\
plants        & $69$       & $17412$       & $2321$        & $3482$        \\
baudio        & $100$      & $15000$       & $2000$        & $3000$        \\
jester        & $100$      & $9000$        & $1000$        & $4116$        \\
bnetflix      & $100$      & $15000$       & $2000$        & $3000$        \\
accidents     & $111$      & $12758$       & $1700$        & $2551$        \\
tretail       & $135$      & $22041$       & $2938$        & $4408$        \\
pumsb\_star   & $163$      & $12262$       & $1635$        & $2452$        \\
dna           & $180$      & $1600$        & $400$         & $1186$        \\
kosarek       & $190$      & $33375$       & $4450$        & $6675$        \\
msweb         & $294$      & $29441$       & $3270$        & $5000$        \\
tmovie        & $500$      & $4524$        & $1002$        & $591$         \\
book          & $500$      & $8700$        & $1159$        & $1739$        \\
cwebkb        & $839$      & $2803$        & $558$         & $838$         \\
cr52          & $889$      & $6532$        & $1028$        & $1540$        \\
c20ng         & $910$      & $11293$       & $3764$        & $3764$        \\
bbc           & $1058$     & $1670$        & $225$         & $330$         \\
ad            & $1556$     & $2461$        & $327$         & $491$         \\
\bottomrule
\end{tabular}
\end{table}

\paragraph{Real World Data.} We consider the standard suite of $20$ binary density estimation benchmark \cite{van2012markov, bekker2015tractable}. These include small to large domains such as \texttt{nltcs} ($16$ variables), up to \texttt{ad} ($1556$ variables). Table \ref{tab:20DEBD} summarizes the number of variables and data points in the train, validation, and test splits for each of the $20$ datasets.

\begin{figure}[ht]
    \centering
    \begin{subfigure}{0.45\linewidth}
        \includegraphics[width=\linewidth]{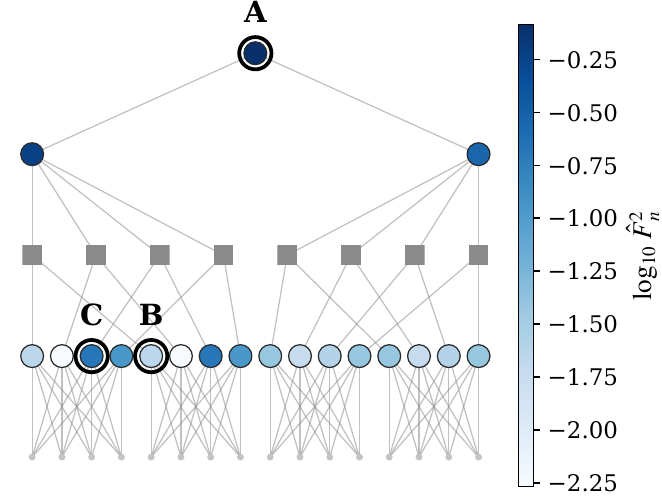}
        \caption{Contextual usage $\widehat{F}_n^2$}
        \label{fig:decomp-large-a}
    \end{subfigure}
    \hfill
    \begin{subfigure}{0.45\linewidth}
        \includegraphics[width=\linewidth]{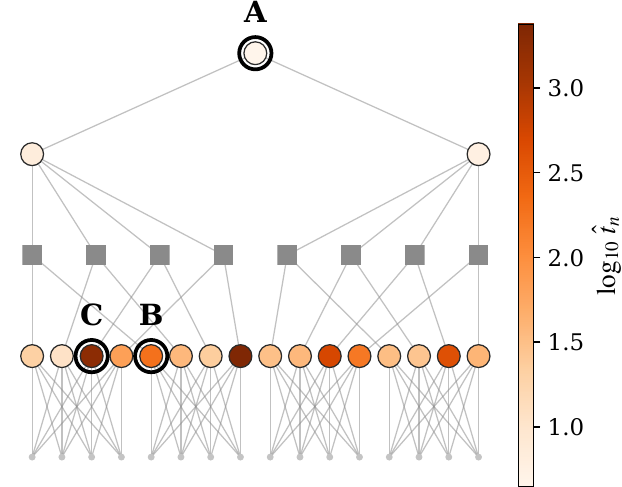}
        \caption{Local curvature $\that$}
        \label{fig:decomp-large-b}
    \end{subfigure}
    \begin{subfigure}{0.45\linewidth}
        \includegraphics[width=\linewidth]{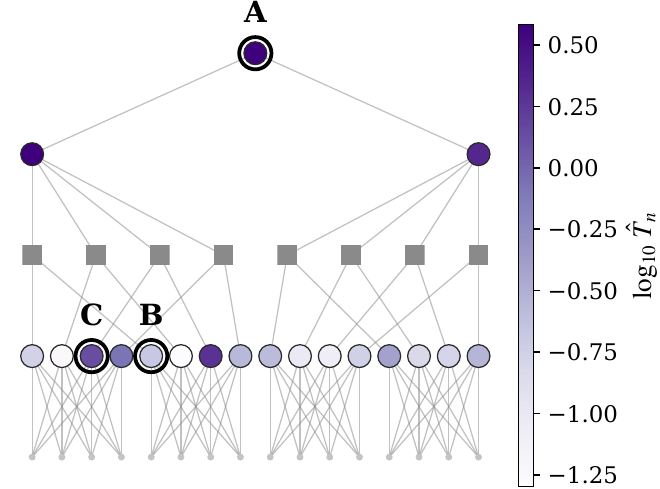}
        \caption{Global contribution $\Th$}
        \label{fig:decomp-large-c}
    \end{subfigure}
    \hfill
    \begin{subfigure}{0.45\linewidth}
        \includegraphics[width=\linewidth]{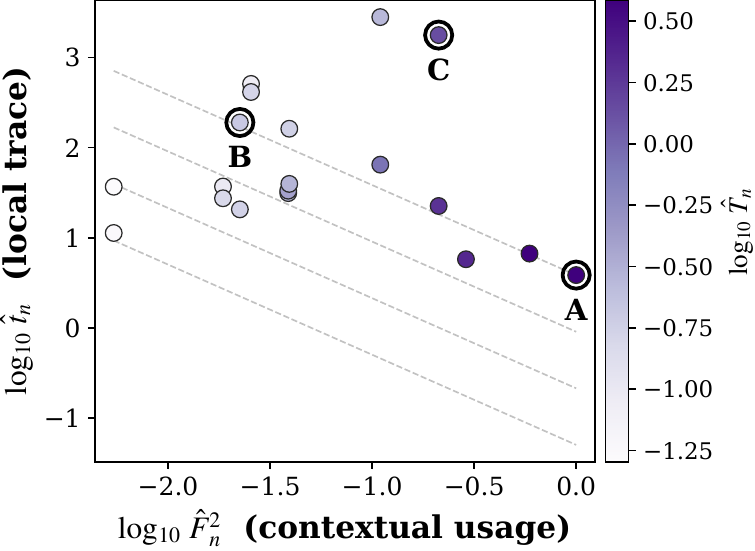}
        \caption{$\Th=\widehat{F}_n^2\,\that$}
        \label{fig:decomp-large-d}
    \end{subfigure}
    \caption{\textbf{The curvature of a trained probabilistic circuit factorizes into usage and local sharpness.}
    Each sum node contributes $\Tnn=\Fn^2\,\tn$ to the NLL Hessian trace. Panels~(a)-(c) show the same circuit, with identical node coordinates, colored by (a)~squared flow $\widehat{F}_n^2$, (b)~local trace $\that$, and (c)~global contribution $\Th$; panel~(d) plots $\log\widehat{F}_n^2$ against $\log\that$ colored by $\log\Th$, with dashed iso-contribution lines of constant $\Th$. The two factors are anticorrelated across the circuit: the locally sharpest nodes (dark in~b) sit at low flow (light in~a) and contribute little globally, while the high-usage root region drives the global trace despite modest local curvature. Nodes~A, B, and~C mark the three regimes: high-usage/low-local~(A), low-usage/high-local~(B), and high on both~(C). }
    \label{fig:decomp-large}
\end{figure}
\subsection{Model Architectures}

\begin{itemize}
    
    \item \noindent \textbf{PyJuice} \cite{LiuICML2024}, for our experiments on the binary density estimation datasets. We use the Hidden Chow-Liu Tree structure \citep{LiuNeurIPS2021,LiuICML2024} 
    which is a generative probabilistic model that extends the classical Chow-Liu tree by introducing latent (hidden) variables to model complex dependencies among observed variables more effectively. The tree topology is learned from data using maximum‐likelihood (Chow-Liu algorithm). The observed variables are then pushed to the leaves, introducing latent variables to occupy the internal nodes, forming a latent tree structure. As a result, the learned structure can vary across datasets, adapting to the underlying statistical relationships. The latent size (\texttt{num\_latent}) refers to the number of states each hidden variable can take, and it serves as a key hyperparameter that controls the model’s capacity.
    We set \texttt{num\_latent}$=100$ for all our binary density estimation datasets.
\end{itemize}

\section{Additional Results}
This section provides additional evidence for the empirical observations and method comparisons reported in the main paper. We first visualize the two components of the curvature decomposition and present a controlled synthetic example. We then report full-dataset results for underfitting, trace concentration, node-selection criteria, and performance across data regimes.

\subsection{Visualizing the Curvature Decomposition}
\label{app:decomp-viz}

Figure~\ref{fig:decomp-large} visualizes the empirical contextual usage, local curvature, and global trace contribution of the sum nodes in a trained circuit. In this example, several nodes with large local curvature receive little flow and therefore contribute weakly to the global trace. Conversely, nodes near the root receive substantial flow and can dominate the global contribution despite moderate local curvature.
The labeled nodes illustrate the ranking reversal characterized in Proposition~\ref{prop:reversal}. Node~A has the smallest local curvature among the three but the largest global contribution because it receives the most flow. Node~B is locally sharper but contributes less globally because its contextual usage is small. The example illustrates why global contribution and local curvature provide different signals for allocating regularization.

\subsection{A 2D Synthetic Data Distribution}
\label{app:synthetic}
We construct a two-dimensional example to illustrate how uniform and locally gated trace regularization can favor different regions of parameter space. The training set contains $50$ samples from a mixture of three isotropic Gaussians with standard deviation $\sigma=0.25$, centered at the vertices of an equilateral triangle. We additionally include eight noise samples distributed around a ring of radius $2.3$, with radial jitter $0.12$. The test set contains $800$ samples drawn only from the Gaussian mixture.
We train a PC with Gaussian input distributions and $16$ sum and input units under three regimes: unregularized maximum likelihood, global trace regularization, and adaptive local-curvature gating. All methods use the same initialization, and both regularized methods use $\mu=0.20$.
Following \citet{li2018visualizing}, we project the optimization trajectories onto the plane spanned by the leading two principal components of the parameter iterates. We then evaluate the train NLL, test NLL, and exact Hessian trace over this plane. Figure~\ref{fig:synthetic-landscape} shows that the three methods converge to different regions. The unregularized model attains the lowest train NLL but higher test NLL and curvature. Global regularization reaches a flatter region but sacrifices both train and test fit. Adaptive gating retains lower curvature while reaching a region with better test fit than uniform global regularization.

\begin{figure}[ht]
    \centering
    \includegraphics[width=\linewidth]{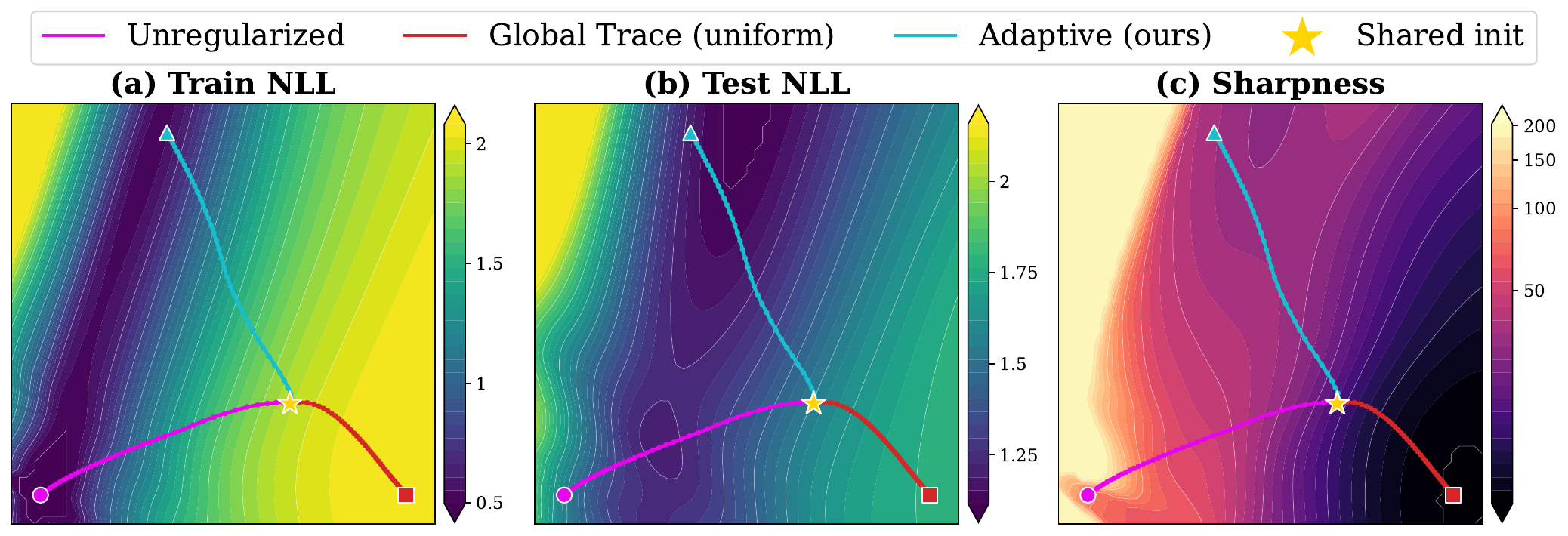}
    \caption{\textbf{Vanilla, global, and adaptive regularization settle in distinct regions of parameter space.}
     Vanilla reaches the lowest train NLL but a sharp, high-curvature region that generalizes poorly. Global regularization flattens curvature indiscriminately and underfits, while adaptive gating flattens only the non-generalizing sharpness, reaching a region with competitive test NLL at markedly lower sharpness.
     }
    \label{fig:synthetic-landscape}
\end{figure}

\subsection{Underfitting at full data regime}
Figure~\ref{fig:DEBD_underfitting} reports the training and test log-likelihood trajectories of the unregularized and global trace-regularized models across the DEBD benchmarks using the full training sets. On most datasets, global regularization lowers both training and test likelihood. The resulting degradation is therefore consistent with reduced model fit rather than a larger train--test gap.

\begin{figure}[ht]
    \centering
    \includegraphics[width=0.8\linewidth]{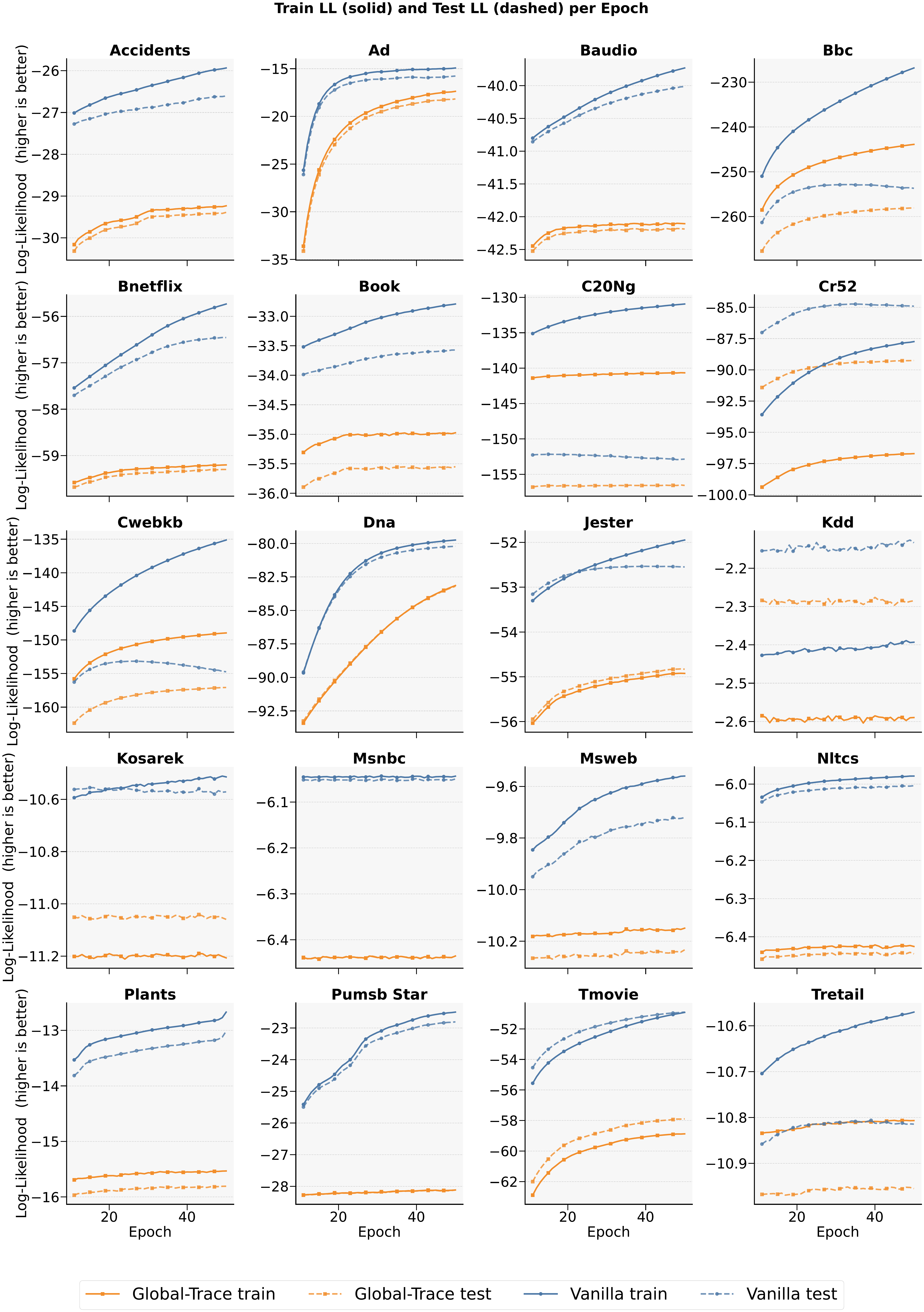}
    \caption{(Underfitting at full data regime)Training and test log-likelihood trajectories for the
vanilla and global trace-regularized models DEBD using the full training data. Global trace regularization reduces the train-test gap but lowers both training and test likelihood, indicating underfitting rather than
improved generalization.}
    \label{fig:DEBD_underfitting}
\end{figure}

\subsection{Sparsity of Global and Local Trace}
Figures~\ref{fig:Gtrace_sparsity} and~\ref{fig:Ltrace_sparsity} compare how global trace contribution and local curvature are distributed across sum nodes. The global contribution is highly concentrated across the evaluated datasets, with a small fraction of nodes accounting for most of the total trace. Local curvature is generally distributed across a broader set of nodes, indicating that contextual flow contributes to the concentration observed in the global trace.

\begin{figure}[ht]
    \centering
    \includegraphics[width=0.9\linewidth]{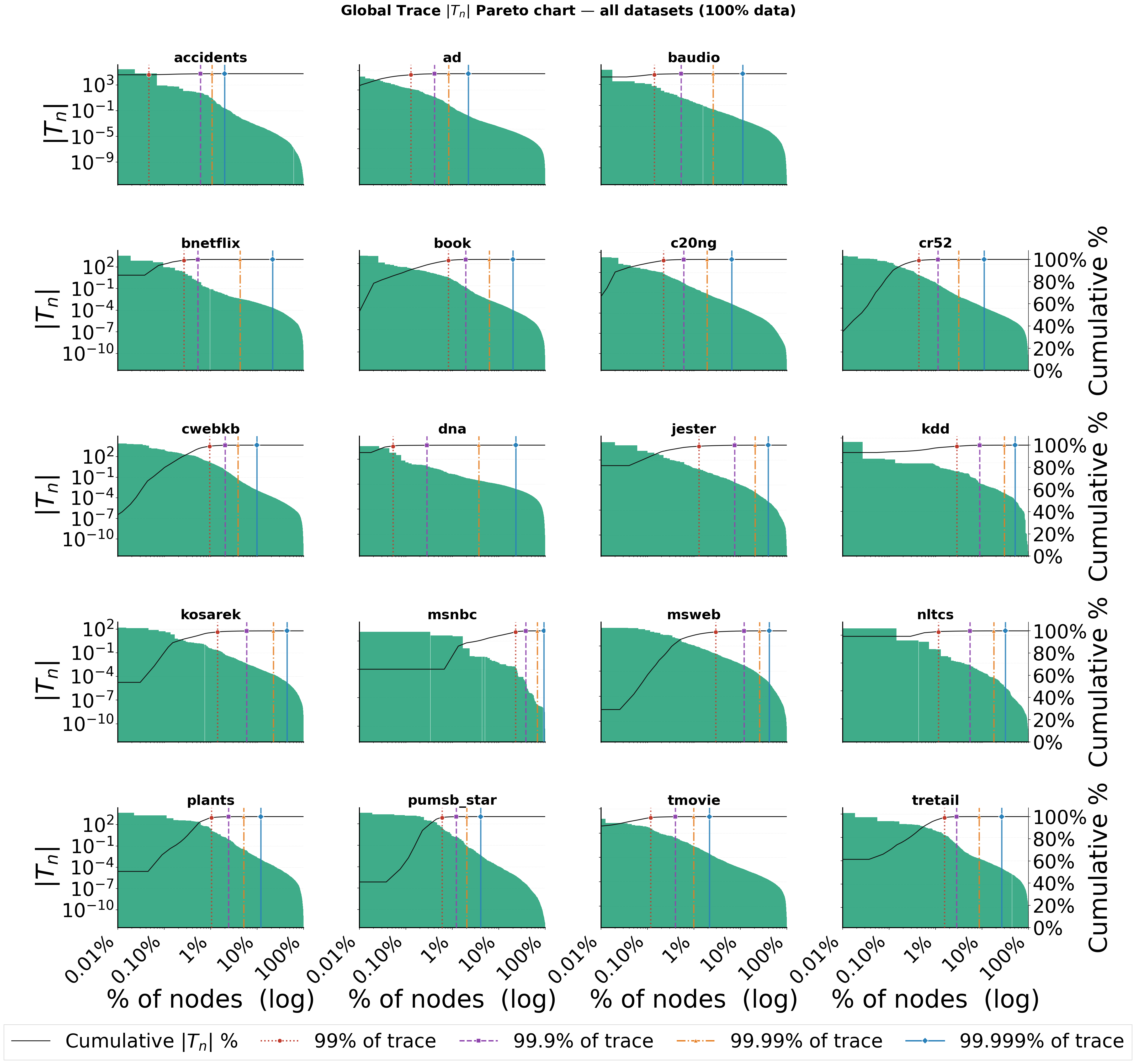}
    \caption{(Sparsity of Global trace) Cumulative share of the total trace as sum nodes are ranked by theirempirical contribution  $(\Th)$. Across datasets, the distribution is highly concentrated: a small set of nodes dominates the global sharpness signal, while most nodes contribute negligibly.}
    \label{fig:Gtrace_sparsity}
\end{figure}

\begin{figure}[ht]
    \centering
    \includegraphics[width=0.9\linewidth]{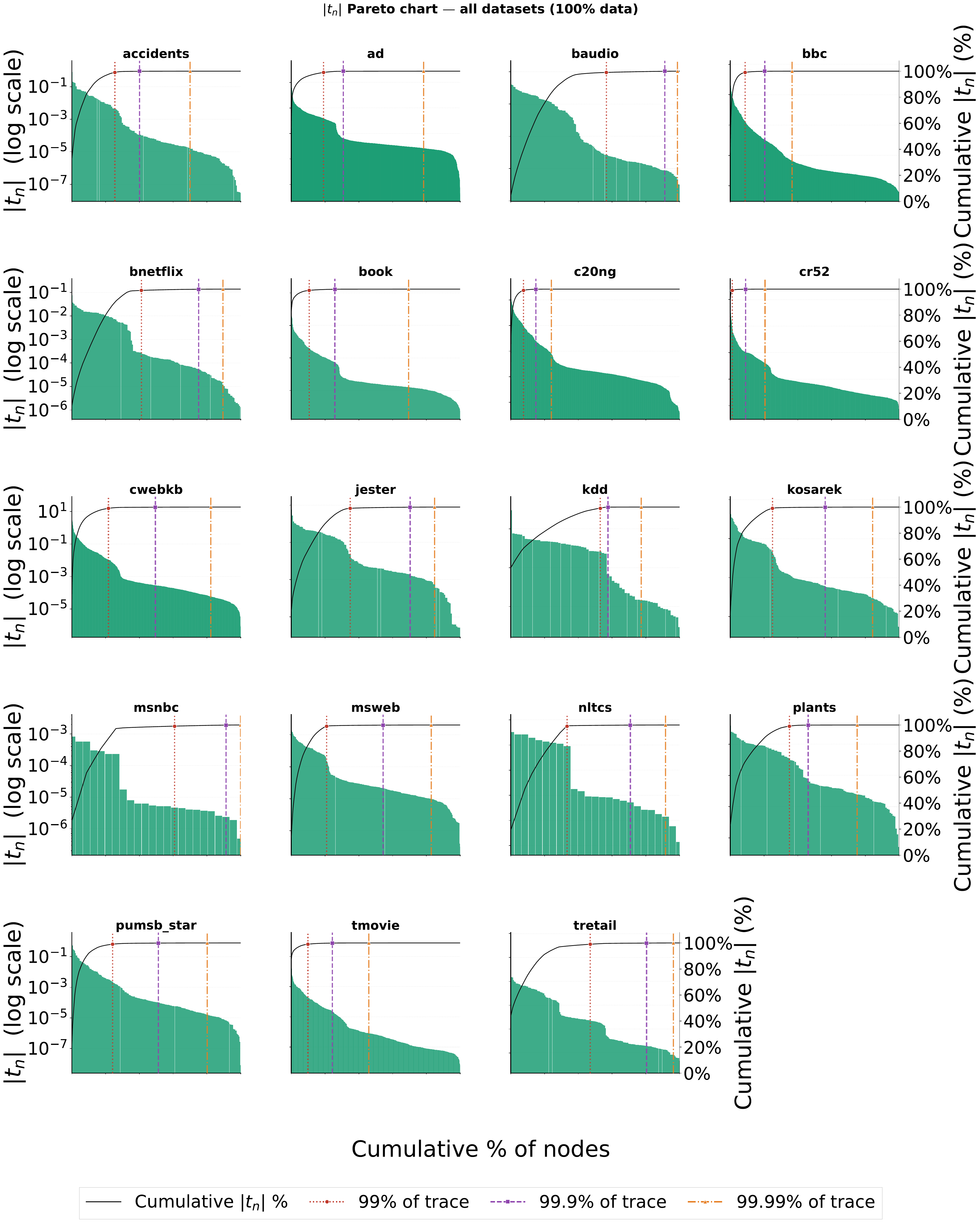}
    \caption{(Sparsity of local trace) Cumulative share of $\sum_n\that$ as nodes are ranked by $\that$. On the evaluated circuits, the
    local trace concentration is less severe than the global contribution, indicating that part of the concentration in $\Th$ is
    introduced by contextual usage.}
    \label{fig:Ltrace_sparsity}
\end{figure}

\subsection{Thresholding based node selection results}
We next compare two criteria for selecting where trace regularization acts. In Figure~\ref{fig:Gtrace_threshold}, nodes are ranked by global contribution. Restricting the penalty to progressively smaller sets of high-contribution nodes generally fails to recover the unregularized likelihood and often degrades it further. Figure~\ref{fig:Ltrace_threshold} instead ranks nodes by local curvature. Relative to global-contribution selection, local-curvature selection more consistently preserves the unregularized fit across the evaluated datasets.

\begin{figure}[ht]
    \centering
    \includegraphics[width=0.9\linewidth]{Figures/Supplementary/test_ll_dot_all_datasets_5x4.pdf}
    \caption{(Thresholding based node selection results) Change in test log-likelihood relative to the unregularized model when trace regularization is applied to all sum nodes or restricted to progressively smaller fractions ranked by $\widehat{T}_n$. Selecting fewer top-ranked nodes generally worsens performance, showing that global contribution is a poor criterion for allocating regularization.}
    \label{fig:Gtrace_threshold}
\end{figure}

\begin{figure}[ht]
    \centering
    \includegraphics[width=0.8\linewidth]{Figures/Supplementary/test_ll_local_dot_all_datasets_5x4.pdf}
    \caption{(Thresholding based node selection results) Change in test log-likelihood relative to the unregularized model when trace regularization is applied to all sum nodes or restricted to progressively smaller fractions ranked by $\that$. Selecting fewer top-ranked nodes to regularize  generally improves the performance, showing that $\that$ is a better criterion for allocating regularization.}
    \label{fig:Ltrace_threshold}
\end{figure}

\subsection{Low Data Regime Experiments}
\begin{table*}[h]
\small
\centering
\caption{Test log-likelihood comparison across data fractions for DEBD. For gated methods, the variant is selected per dataset/fraction by highest mean validation LL.}

\setlength{\tabcolsep}{1pt}   
\renewcommand{\arraystretch}{1}

\begin{tabular}{lcccccccccccc}
\toprule
Dataset &  \multicolumn{3}{c}{25\% Data} & \multicolumn{3}{c}{50\% Data} & \multicolumn{3}{c}{100\% Data} \\
\cmidrule(lr){2-4}\cmidrule(lr){5-7}\cmidrule(lr){8-10}\cmidrule(lr){11-13}
 & Vanilla & Global Trace & Best Gated & Vanilla & Global Trace & Best Gated & Vanilla & Global Trace & Best Gated \\
\midrule
accidents  & -28.59 & -29.83 & -28.64 & -27.13 & -29.84 & -27.12 & -26.63 & -29.69 & -26.64 \\
ad &   -28.55 & -29.64 & -28.86 & -20.51 & -22.07 & -20.47 & -18.30 & -20.37 & -18.22 \\
baudio & -41.13 & -42.46 & -41.16 & -39.99 & -42.40 & -40.01 & -39.63 & -42.45 & -39.63 \\
bbc & -379.30 & -281.33 & -377.54 & -312.26 & -262.64 & -309.58 & -269.69 & -256.43 & -269.39 \\
bnetflix & -59.06 & -59.87 & -59.24 & -56.85 & -59.83 & -56.95 & -56.22 & -59.67 & -56.21 \\
book & -37.91 & -37.48 & -37.45 & -35.44 & -36.75 & -35.29 & -34.37 & -36.31 & -34.23 \\
c20ng & -162.01 & -160.21 & -161.97 & -154.92 & -158.67 & -154.94 & -151.97 & -157.60 & -152.04 \\
cr52 & -107.98 & -103.75 & -106.04 & -102.79 & -102.07 & -101.00 & -99.35 & -102.15 & -100.26 \\
cwebkb & -204.80 & -164.74 & -202.88 & -171.74 & -159.32 & -170.11 & -154.78 & -157.34 & -154.89 \\
dna & -93.73 & -82.61 & -93.44 & -90.35 & -81.76 & -90.07 & -87.70 & -81.35 & -87.11 \\
jester & -57.38 & -55.65 & -57.43 & -54.05 & -55.90 & -54.26 & -52.82 & -55.81 & -52.88 \\
kdd  & -2.22 & -2.37 & -2.21 & -2.22 & -2.36 & -2.21 & -2.23 & -2.37 & -2.22 \\
kosarek  & -10.83 & -11.11 & -10.79 & -10.64 & -11.06 & -10.67 & -10.59 & -11.09 & -10.59 \\
msnbc & -6.10 & -6.46 & -6.08 & -6.11 & -6.47 & -6.08 & -6.12 & -6.46 & -6.08 \\
msweb & -10.00 & -10.45 & -9.97 & -9.82 & -10.39 & -9.81 & -9.72 & -10.40 & -9.74 \\
nltcs  & -6.09 & -6.50 & -6.10 & -6.02 & -6.47 & -6.02 & -6.00 & -6.48 & -6.00 \\
plants & -12.97 & -15.90 & -12.96 & -12.77 & -15.97 & -12.77 & -12.67 & -15.98 & -12.66 \\
pumsb\_star & -23.82 & -28.75 & -23.65 & -22.79 & -28.35 & -22.77 & -22.79 & -28.35 & -22.78 \\
tmovie & -51.23 & -50.51 & -50.89 & -44.55 & -49.57 & -44.58 & -42.17 & -48.84 & -42.11 \\
tretail & -11.31 & -11.03 & -11.27 & -10.96 & -11.00 & -10.98 & -10.84 & -10.98 & -10.85 \\
\bottomrule
\end{tabular}
\label{tab:summary}
\end{table*}
Table~\ref{tab:summary} compares the unregularized model, global trace regularization, and the best gated variant across the evaluated data fractions. For each dataset and fraction, the gated variant is selected using validation log-likelihood. The results show that global regularization can improve performance for some datasets, particularly when data are limited, but frequently degrades likelihood as the available data increase. The gated variants more consistently preserve the fit while retaining gains on datasets where sharpness regularization remains beneficial.

\end{document}